\documentclass{article}

\usepackage{PRIMEarxiv}
\pdfoutput=1

\usepackage{graphicx}
\usepackage{mathptmx} 
\usepackage{natbib}
\usepackage{amsmath}
\usepackage{amsthm}
\usepackage{amsfonts}
\usepackage{amssymb}
\usepackage{bm}
\usepackage{url}
\usepackage{comment}

\usepackage{subcaption}
\usepackage{algorithmic}
\usepackage[ruled,vlined]{algorithm2e}
\usepackage{multirow}
\usepackage{dsfont}
\usepackage{epstopdf}
\usepackage{mathptmx}
\DeclareMathAlphabet{\mathcal}{OMS}{cmsy}{m}{n}

\usepackage[utf8]{inputenc} 
\usepackage[T1]{fontenc}    
\usepackage{hyperref}       
\usepackage{booktabs}       
\usepackage{nicefrac}       
\usepackage{microtype}      
\usepackage{lipsum}
\usepackage{fancyhdr}       

\newcommand{\myx}{\bm{x}} 
\newcommand{\Weight}{D} 
\newcommand{\weight}{\bm{d}} 
\newcommand{\weighti}{d} 
\newcommand{\tp}{\mathsf{T}}
\newcommand{\indi}{\mathds{1}} 
\newcommand{\iter}{t} 
\newcommand{\NumFea}{P} 
\newcommand{\subdim}{P} 
\newcommand{\ncluster}{K} 
\newcommand{\ridgepen}{\xi} 
\newcommand{\lab}{c} 

\usepackage{xcolor}

\def\R{\mathbb{R}}
\def\X{\mathcal{X}}
\def\S{\mathcal{S}}
\def\myb{\bm{\beta}}
\def\Xmi{\hat{X}_{\mathcal{I}}}
\def\tx{\hat{\bm{x}}}
\def\bx{\bar{\bm{x}}}
\def\I{\mathcal{I}}
\def\U{\mathcal{U}}

\def\L{\mathcal{L}}

\DeclareMathOperator*{\argmin}{arg\,min}
\DeclareMathOperator*{\argmax}{arg\,max}

\newtheorem{lemma}{Lemma}
\newtheorem{proposition}{Proposition}

\graphicspath{{media/}}     

\title{Weighted Sparse Subspace Representation\\
\large A Unified Framework for Subspace Clustering, Constrained Clustering, and Active Learning\\
}

\author{
  Hankui Peng \\
  Department of Applied Mathematics and Theoretical Physics \\
  University of Cambridge \\
  Cambridge, UK\\
  \texttt{hp467@cam.ac.uk} \\
   \And
  Nicos G. Pavlidis \\
  Department of Management Science \\
  Lancaster University \\
  Lancaster, UK\\
  \texttt{n.pavlidis@lancaster.ac.uk} \\
}

\begin{document}
\maketitle

\begin{abstract}
Spectral-based subspace clustering methods have proved successful in many
challenging applications such as gene sequencing, image recognition,
and motion segmentation. In this work, we first propose a novel
spectral-based subspace clustering algorithm that seeks to represent
each point as a sparse convex combination of a few nearby points. We
then extend the algorithm to constrained clustering and active learning
settings. Our motivation for developing such a framework stems from the
fact that typically either a small amount of labelled data is available
in advance; or it is possible to label some points at a cost. The
latter scenario is typically encountered in the process of validating a cluster assignment. 
%
%
Extensive
experiments on simulated and real data sets show that the proposed
approach is effective and competitive with state-of-the-art methods.
\end{abstract}


\section{Introduction}\label{wssr_intro}

In many challenging real-world applications involving the grouping of
high-dimensional data, different clusters can be well approximated as lower
dimensional subspaces.  This is the case for example in gene
sequencing~\citep{mcwilliams2014subspace}, face
clustering~\citep{elhamifar2013sparse}, motion
segmentation~\citep{rao2009motion}, and text mining~\citep{peng2018subspace}. The problem of simultaneously estimating
the subspace corresponding to each cluster and partitioning a group of points
into clusters according to these subspaces is called \emph{subspace
clustering}~\citep{vidal2011subspace}.


%
Spectral methods for subspace clustering have demonstrated excellent
performance in numerous real-world applications \citep{liu2012robust,
lu2012robust, elhamifar2013sparse, li2015structured, huang2015new}. These
methods construct an affinity matrix for spectral clustering by solving an
optimisation problem that aims to approximate each point through a linear
combination of other points from the same subspace.
%
%
%
In this paper we first propose a method called \emph{Weighted Sparse Simplex
Representation}~(WSSR). Our method is based on the Sparse Simplex
Representation (SSR) of \cite{huang2013new},
in which each point is approximated through a convex combination of other points. This method was not proposed as a subspace clustering method, but rather for modelling the brain anatomical and genetic networks. 
We modify SSR to ensure
that each point is approximated through a sparse convex combination of nearby
neighbours, and thus obtain an algorithm that is effective for the subspace clustering problem.

Due to the complete lack of labelled data clustering methods rarely achieve
perfect performance. In practice, it is often the case that a small amount of
data is either available in advance, or can be obtained at a certain cost.  The
latter scenario is very common during the process of validating a clustering
result. If labelled data are available the clustering algorithm should be able
to accommodate this external information (constrained clustering). Even more
interesting is the active learning setting in which the choice of points to be
labelled is a part of the learning problem. In this case the algorithm should
choose to query the labels of points so as to maximise the quality of the
overall model.
%
%
%
In this work, we propose an iterative active learning and constrained
clustering framework for WSSR. We draw upon the work of \cite{peng2019subspace}
to select informative points to query for subspace clustering.  The cluster
assignment is updated after obtaining the labels of these points. Our proposed
constrained clustering approach is guaranteed to produce an assignment that is
consistent with all the available labels.


%
%


The rest of this paper is organised as follows. In Section \ref{wssr_rw}, we
discuss some relevant existing literature in the areas of subspace clustering,
constrained clustering, and active learning. In Section \ref{wssr_wssr}, we
propose the problem formulation of the Weighted Sparse Simplex Representation
(WSSR), discuss its properties, and present an approach for solving the
problem. In Section~\ref{sec_wssr_si}, we propose an integrated active learning
and constrained clustering framework. We demonstrate the effectiveness of our
proposed methodology on synthetic and real data in Section~\ref{sec_wssr_syn}
and \ref{sec_wssr_real}, and provide concluding remarks in
Section~\ref{wssr_conclusions}.

\section{Related Work}\label{wssr_rw}

%
The linear subspace clustering problem can be defined as follows. A collection of $N$
data points $\X = \left\{\myx_{i} \right\}_{i=1}^{N} \subset \R^\NumFea$ is
drawn from a union of $\ncluster$ linear subspaces $\left\{\mathcal{S}_{k}
\right\}_{k=1}^{\ncluster}$ with added noise. Each subspace can be defined as,
%
%
\begin{equation}\label{eq_subspace}
%
%
\S_{k}=\left\{\myx \in \R^{\NumFea} \,|\, \myx= V_{k}\bm{y} \right\},
\;\; \text{for} \;\; k=1, \ldots, \ncluster,
\end{equation}
where $V_{k} \in \R^{\NumFea \times \subdim_{k}}$, with $1 \leqslant
\subdim_{k}<\NumFea$, is a matrix whose columns constitute a basis for
$\S_{k}$, and $\bm{y}\in\R^{\subdim_{k}}$ is the representation of
$\bm{x}$ in terms of the columns of $V_{k}$. 
The goal of subspace clustering is to find the number of subspaces $\ncluster$;
the subspace dimensions $\left\{\subdim_{k} \right\}_{k=1}^{\ncluster}$; a
basis $\left\{V_{k} \right\}_{k=1}^{\ncluster}$ for each subspace; and finally
the assignments of the points in $\X$ to clusters. 
%
%
%
A natural formulation of this problem is
%
%
as the minimisation of the {\em reconstruction error}\/,
%
%
\begin{equation}\label{eq:KSCobjective}
%
%
\sum_{i=1}^{N}\min_{V_1, \ldots, V_\ncluster} \left\{    \min_{k=1,\ldots,K} \|\myx_i - V_{k} V_{k}^\tp \myx_i\|_2^2  \right\}.
\end{equation}
$K$-subspace clustering (KSC)~\citep{bradley2000k} 
is an iterative algorithm 
to solve the problem in~\eqref{eq:KSCobjective}.
Like the classical $K$-means clustering, KSC
alternates between estimating the subspace bases (for
a fixed cluster assignment),
and assigning points to clusters (for a fixed
set of bases). 
However, iterative algorithms
%
are very sensitive to initialisation 
and most often converge to poor local minima~\citep{lipor2017leveraging}. 


Currently the most effective approach to subspace clustering is through
spectral-based methods~\citep{lu2012robust,
elhamifar2013sparse, hu2014smooth, you2016scalable}.
Spectral-based methods consist of two steps: first an affinity matrix is estimated, and
then normalised spectral clustering~\citep{ng2002spectral} is applied to this affinity matrix.
The affinity matrix is constructed by exploiting the {\em self-expressive}
property: any $\myx_{i} \in \S_k$ 
can be expressed as a linear combination of $\subdim_{k}$ other points
from~$\S_k$.
Thus, for each $\myx_i \in \X$ they first solve a convex optimisation problem 
of the form,
\begin{equation}\label{eq_selfexpress}
\bm{\beta}_{i}^\star = \min_{\bm{\beta}_i \in \R^{N-1}}  \left\| \myx_i - X_{-i} \bm{\beta}_i \right\|_{p} + \rho \left\|\bm{\beta}_i \right\|_{q}, 
\end{equation}
where $X_{-i} = \left[\myx_{1},\ldots, \myx_{i-1},\myx_{i+1}\,\ldots,\myx_{N} \right] \in \R^{P\times (N-1)}$ is a
matrix whose columns correspond to the points in $\X \backslash \{\myx_i\}$; and $\rho>0$ is a penalty parameter.
The first term in the
objective function quantifies the error of approximating 
$\myx_i$ through $X_{-i}  \myb_{i}$.
%
%
The penalty (regularisation) term is included to promote solutions
in which $\beta_{ij}^\star$ is small (and ideally zero) if
$\myx_{j}$ belongs to a different subspace than $\myx_{i}$.
%
%
%
%
%
%
%
%
After solving the problem in~\eqref{eq_selfexpress} for each $\myx_i \in \X$,
%
%
the affinity matrix is typically defined as $A = \left(|B|+|B|^\tp \right)/2$, where
$B = [\bm{\beta}_1^\star, \ldots, \bm{\beta}_N^\star]$.

%
%
%
Least Squares Regression (LSR)~\citep{lu2012robust} uses the $L_2$-norm for both the approximation error, and the regularisation term
($p=q=2$). Smooth Representation Clustering (SMR)~\citep{hu2014smooth} 
also uses the $L_2$-norm on the approximation
error, while the penalty term is given by $\|L^{1/2}\bm{\beta}_i\|_2^2$ in which
$L$ a positive definite Laplacian matrix constructed from pairwise similarities.
The main advantage of using the $L_2$-norm is that the optimisation problem has a
closed-form solution. However the resulting coefficient vectors are dense and hence
the affinity matrix contains connections between points from different subspaces.
The most prominent spectral-based subspace clustering algorithm is Sparse
Subspace Clustering (SSC)~\citep{elhamifar2013sparse}. 
In its most general formulation, SSC accommodates the possibility that
points from each subspace are contaminated by both noise and sparse outlying entries.
In SSC, $\bm{\beta}_i^\star$ is the solution to the following problem,
\begin{align}\label{eq:SSC}
\min_{\bm{\beta}_i \in \R^{N-1}} & \; \|\bm{\beta}_i\|_1 + \rho_\eta \|\bm{\eta}_i\|_1 + \frac{\rho_z}{2} \| \bm{z}_i\|_2^2,\\
\text{s.t.} & \; \myx_i = X_{-i} \bm{\beta}_i + \bm{\eta}_i + \bm{z}_i. \nonumber  
\end{align}
SSC therefore decomposes the approximation error
into two components ($\bm{\eta}_i$ and $\bm{z}_i$), which are measured with different norms.
%
Following the success of SSC, several variants have been proposed, including
SSC with Orthogonal Matching Pursuit (SSC-OMP)~\citep{you2016scalable}, 
Structured Sparse Subspace Clustering (S3C)~\citep{li2017structured}, and
Affine Sparse Subspace Clustering (ASSC)~\citep{li2018geometric}.

The method most closely connected to our approach is
the Sparse Simplex Representation (SSR) algorithm, proposed by
\cite{huang2013new} for the modelling of brain networks.
SSR solves the problem in~\eqref{eq_selfexpress} using $p=2$ and $q=1$ with the additional constraint
that the coefficient vector has to lie in the $(N-1)$-dimensional unit simplex 
$\bm{\beta}_i \in \Delta^{N-1} = \{\myb \in \R^{N-1} \,|\, \myb \geqslant 0, \myb^\tp \bm{1}=1\}$. 
Since SSR approximates $\myx_i$ through a convex combination of other points,
the coefficients have a probabilistic interpretation. 
However, SSR induces no regularisation since $\|\myb_i\|_1=1$ for all
$\bm{\beta}_i \in \Delta^{N-1}$, hence coefficient vectors are dense.

We next provide a short overview of clustering with external information,
called {\em constrained clustering}~\citep{basu2008constrained}, and active
learning.
Due to space limitations, we only mention the work that is most closely related
to our problem.
In constrained clustering, the external information can be either in the form
of class labels or as pairwise ``must-link'' and ``cannot-link'' constraints.
Spectral methods for constrained clustering incorporate this information by
modifying the affinity matrix.  Constrained Spectral Partitioning (CSP)
\citep{wang2010flexible} introduces a pairwise constraint matrix and solves a
modified normalised cut spectral clustering problem. Partition Level
Constrained Clustering (PLCC)~\citep{liu2018partition} forms a pairwise
constraint matrix through a side information matrix which is included as a
penalty term into the normalised cut objective.
Constrained Structured Sparse Subspace Clustering
(CS3C)~\citep{li2017structured} 
%
%
is specifically designed for subspace clustering. CS3C incorporates a side
information matrix that encodes the pairwise constraints into the formulation
of S3C. The algorithm alternates between solving for the coefficient matrix and
solving for the cluster labels.
Constrained clustering algorithms that rely exclusively on modifying the
affinity matrix cannot guarantee that all the constraints will be satisfied.
CS3C+~\citep{li2018constrained} is an extension of CS3C that applies 
constrained $K$-means algorithm~\citep{wagstaff2001constrained} within the
spectral clustering stage, to ensure constraints are satisfied.

In active learning the algorithm controls the choice of points for which
external information is obtained.
The majority of active learning techniques are designed for supervised methods,
and little research has considered the problem of active learning for subspace
clustering~\citep{lipor2015margin,lipor2017leveraging,peng2019subspace}.
\cite{lipor2015margin} propose two active strategies.  The first queries the
point(s) with the largest reconstruction error to its allocated subspace. The
second queries the point(s) that is maximally equidistant to its two closest
subspaces. 
\cite{lipor2017leveraging} extend the second strategy for spectral clustering
by setting the affinity of ``must-link'' and ``cannot-link'' pairs of points to
one and zero, respectively.
Both strategies by~\cite{lipor2015margin} are effective in identifying
mislabelled points. However, correctly assigning these points is not guaranteed
to maximally improve the accuracy of the estimated subspaces, hence the overall
quality of the clustering.
\cite{peng2019subspace} propose an active learning strategy for sequentially
querying point(s) to maximise the decrease of the overall reconstruction error
in~\eqref{eq:KSCobjective}.

\section{Weighted Sparse Simplex Representation} \label{wssr_wssr}

In this section, we describe the proposed spectral-based subspace clustering
method, called Weighted Sparse Simplex Representation (WSSR). Here we describe the
WSSR algorithm under the assumption that no labelled data is available. 
The constrained clustering version described in the next section 
accommodates the case of having a subset of labelled observations 
at the start of the learning process.

Let $\weighti_{ij} \geqslant 0$ denote a measure of dissimilarity between
$\myx_i, \myx_j \in \X$, and $\I$ the set of indices of all points
in $\X\backslash \{i\}$ with finite dissimilarity to $\myx_i$, that is
$\I = \{1\leqslant j \leqslant N \,|\, \weighti_{ij} < \infty, \; j \neq i \}$.
%
In WSSR, the coefficient vector for
each~$\myx_{i}\in\mathcal{X}$ is the solution to the following convex
optimisation problem, 
\begin{align}\label{eq_wssr_en}
\bm{\beta}_i^\star = & \argmin_{\bm{\beta}_{i}}  \frac{1}{2} \left\| \myx_{i} - \Xmi \bm{\beta}_{i} \right\|_{2}^{2} +\rho\left\|\Weight_\I \bm{\beta}_{i} \right\|_{1} +\frac{\ridgepen}{2} \|\Weight_\I \bm{\beta}_{i} \|_{2}^{2} \\ 
& \text{s.t.} \quad \bm{\beta}_{i}^{\tp}\bm{1} = 1, \quad\bm{\beta}_{i} \geqslant \bm{0}, \nonumber
\end{align}
where $\rho, \ridgepen>0$ are penalty parameters, 
$\Xmi \in \R^{P \times |I|}$ is a matrix whose columns are the scaled versions of the points in $\X_\I$, and
$\Weight_\I = \text{diag}(\bm{\weight}_{\I})$ is a diagonal 
matrix of finite pairwise dissimilarities between $\myx_i$ and the points in $\X_\I$. 
We first outline our motivation for the choice of penalty function, and then
discuss the definition of $\Xmi$ and
the choice of $\weighti_{ij}$ in the next paragraph.
%
%
%
The use of both an $L_{1}$ and an $L_{2}$-norm in~\eqref{eq_wssr_en} is motivated by the elastic net
formulation~\citep{zou2005regularization}. 
%
%
%
%
The $L_{1}$-norm penalty promotes solutions in which coefficients of dissimilar points 
are zero.
%
%
The $L_{2}$-norm penalty encourages what is known as the {\em grouping effect}\/: 
%
%
if a group of points induces a similar approximation error, then either all 
points in the group are represented in $\myb_i^\star$, or none is.
This is desirable for subspace clustering, because
the points in such a group should belong
to the same subspace. In this case, if this subspace is different from the one 
$\myx_i$ belongs to, then all points should be assigned a coefficient of zero.
%
%
If instead the group of points are from the same subspace as $\myx_i$, then 
it is beneficial to connect $\myx_i$ to all of them, because this increases
the probability that points from each subspace will belong
to a single connected component of the graph defined by the affinity matrix~$A$. 
%
%

%
Spectral subspace clustering algorithms commonly normalise the points in $\X$
to have unit $L_2$-norm prior to estimating the coefficient
vectors~\citep{elhamifar2013sparse,you2016scalable}.
The simple example in Figure~\ref{fig_3d_example} illustrates that this
normalisation tends to increase cluster separability. In the following, we denote $\bx_{i} = \myx_i / \|\myx_i\|_2$.
%
%
\begin{figure}[ht!] \centering \begin{subfigure}{.49\textwidth}
\includegraphics[width=\textwidth, trim={3cm, 8cm, 3cm, 9cm}]{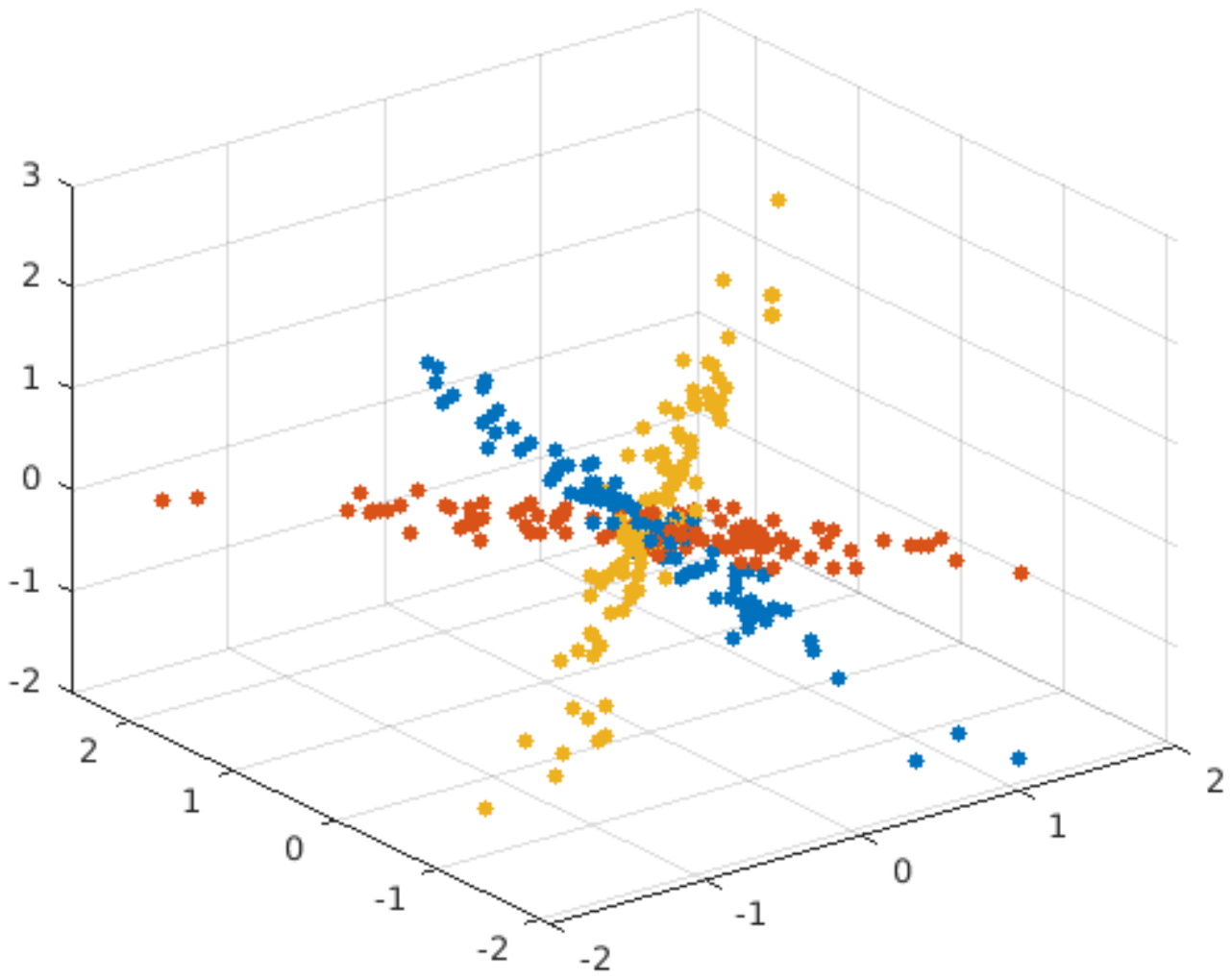}
\end{subfigure} \begin{subfigure}{.49\textwidth}
\includegraphics[width=1.07\textwidth, trim={3cm, 9.8cm, 3cm,
9.8cm}]{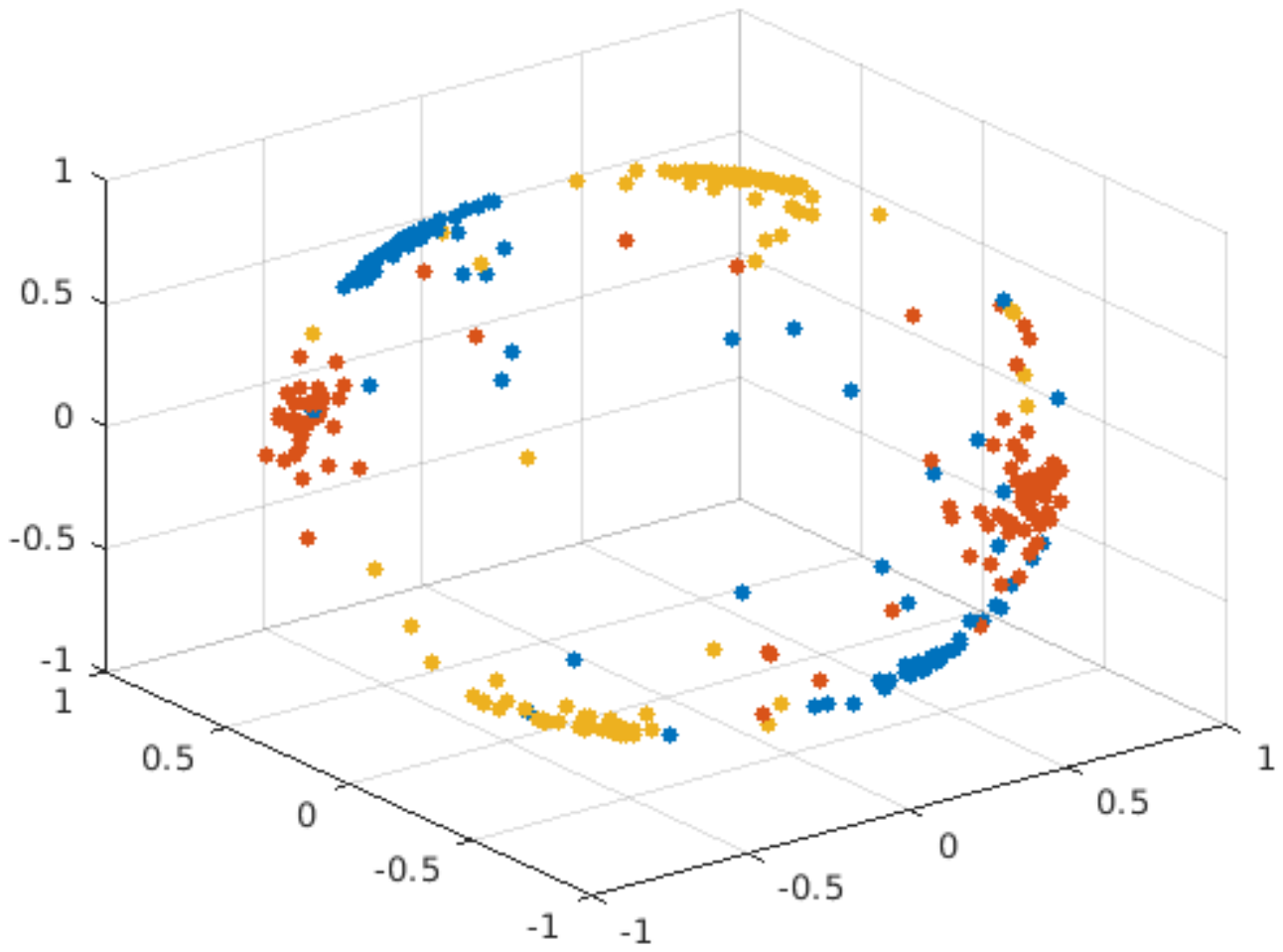}
\end{subfigure} \caption{An illustration of the data normalisation step, and the rationale for using the inverse cosine similarity as the dissimilarity measure.
\textbf{Left:} The original data points. \textbf{Right:} The data points that
have been normalised to lie on the unit sphere.}\label{fig_3d_example}
\end{figure}
However, projecting the data onto the unit sphere has important implications
for the WSSR problem. 
In~\eqref{eq_wssr_en} we want the two conflicting objectives of minimising the
approximation error and selecting a few nearby points to be separate.
Fig.~\ref{fig_StretchingPoints} contains an example that shows that this is not
true after projecting onto the unit sphere.
%
%
%
%
%
%
%
\begin{figure}[htbp!]
\centering
\includegraphics[height=.48\textwidth, trim={3cm 0 0 0},
clip]{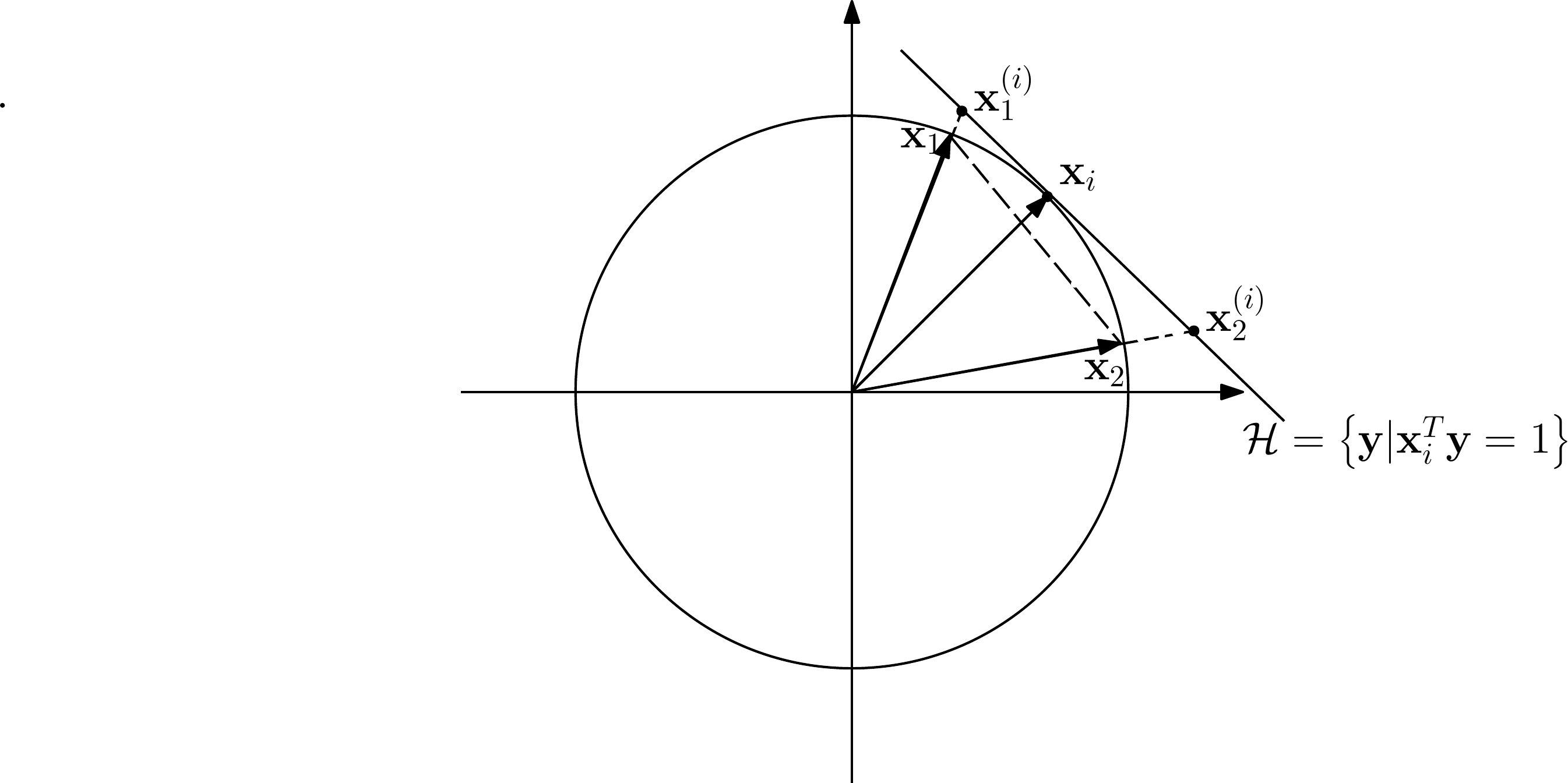}
\caption{A geometric illustration of the necessity for stretching points in
$X$.}\label{fig_StretchingPoints}
\end{figure}
%
%
In Fig.~\ref{fig_StretchingPoints}, the point closest to $\bx_i$ on the unit sphere is $\bx_{1}$.
The direction of $\bx_{i}$ can be perfectly approximated by a convex
combination of $\bx_{1}$ and $\bx_{2}$, but all convex combinations $\alpha
\bx_{1} + (1-\alpha) \bx_{2}$ with $\alpha \in (0,1)$ have an $L_2$-norm less
than one.
%
%
Since the cardinality of $\myb_i$ affects the length of the approximation, it affects both the penalty term and the approximation error. 
%
%
%
%
%
A simple solution to address this problem is 
%
%
%
to scale every point $\myx_j$ with $j \in \I$ such that $\tx_j^{i} = t^{i}_j
\myx_j$ lies on the hyperplane perpendicular to the unit sphere at~$\bx_i$,
$\tx_j^{i} \in  \{\hat{\myx} \in \R^\NumFea \;|\; \hat{\myx}^\tp \bx_i=1\}$.
%
%
Note that this implies that if $\myx_j^\tp \myx_i<0$, then $t^{i}_j$ is
negative.
An inspection of Figure~\ref{fig_3d_example} suggests that this is sensible.
An appropriate measure of pairwise dissimilarity given the aforementioned
preprocessing steps is the inverse cosine similarity,
\begin{equation}\label{eq_dij}
\weighti_{ij}= \|\myx_{i}\|_2 \|\myx_{j}\|_2 /(\myx_i^{\tp} \myx_{j}) =
|\bx_{i}^{\tp}\bx_{j}|^{-1}.
\end{equation}
%
%
%
%
%
Since $\weighti_{ij}$ is infinite when $\myx_j^\tp \myx_i = 0$,
such points could never be assigned a non-zero coefficient, therefore they 
%
%
are excluded from $\Xmi$.  
%

We now return to the optimisation problem in~\eqref{eq_wssr_en}. Due to the constraint $\myb_i \in \Delta^{|\I|}$ and the fact that $\weight_{\I} >0$,
$\|D_\mathcal{I} \myb_i\|_1 = \weight_\I^\tp \myb_i$.
This implies that the objective function is a quadratic,
and the minimisation problem in~\eqref{eq_wssr_en} is equivalent
to the one below,
%
%
\begin{align}\label{eq:nqp}
%
\min_{\myb_{i}} & \frac{1}{2} \myb_i^\tp ( \Xmi^\tp \Xmi + 
\ridgepen D_\mathcal{I}^2) \myb_{i} + (\rho \weight_{\mathcal{I}}- \Xmi^\tp \tx_i)^\tp \myb_i.
\end{align}
Moreover, $f(\myb_i)$ is guaranteed to be strictly convex whenever $\ridgepen>0$.
%
%
%
Therefore WSSR corresponds to the following quadratic programme (QP),
\begin{align}\label{eq_wssr_qp}
\bm{\beta}_i^\star = & \min_{\myb_i} f(\myb_i), \;\; \text{s.t.} \;\; \myb_i \geqslant 0, \; \; \myb_i^\tp \bm{1} = 1.
%
%
\end{align}
%
%

The choice of $\rho$ in~\eqref{eq_wssr_qp} is critical to obtain an affinity matrix that
accurately captures the cluster structure.
%
The ridge penalty parameter, $\ridgepen$, is typically assigned with
a small value, e.g. $10^{-4}$~\citep{gaines2018algorithms}.
%
%
For ``large'' $\rho$, the optimal solution
assigns a coefficient of one to the nearest neighbour of $\bx_i$ 
and all other coefficients are zero. This is clearly
undesirable.
%
%
Setting this parameter is complicated by the fact that
an appropriate choice of $\rho$ differs for each~$\bx_i$.
Lemma~\ref{proof_wssr_nec} is a result which can be used
to obtain a lower bound on 
$\rho$ such that the solution of~\eqref{eq_wssr_qp} 
is the ``nearest neighbour'' approximation.
%
%
%
%
The proof of this lemma, as well as a second lemma which
establishes its geometric interpretation
can be found in Appendix~\ref{appd:nec_suf}.

\begin{lemma}\label{proof_wssr_nec}
Let $(j)$ denote the index of the $j$-th nearest neighbour of $\bx_i$, and $\tx_{i}^{(j)}$ denote the $j$-th nearest neighbour of $\bx_{i}$. 
Assume that $\tx^{(1)}_{i}$ is unique and that $\tx^{(1)}_{i} \neq \bx_i$. We also
assume that the pairwise dissimilarities satisfy:
\begin{align*}
\| \bx_i - \tx^{(j)}_i \|_2  > \| \bx_i - \tx^{(k)}_i \|_2  & \Rightarrow \weighti_{ij} > \weighti_{ik},\\
\| \bx_i - \tx^{(j)}_i \|_2  = \| \bx_i - \tx^{(k)}_i \|_2  &\Rightarrow \weighti_{ij} = \weighti_{ik}.
\end{align*}
If
\[
\bm{e}_{1}=[1,0,\ldots,0 ]^{\tp} = \argmin_{\myb_i \in  \Delta^{|\mathcal{I}|}} \; \frac{1}{2} \myb_i^\tp ( \Xmi^\tp \Xmi + \ridgepen D_\mathcal{I}^{\tp}D_\mathcal{I}) \myb_{i} + (\rho - \Xmi^\tp \bx_i)^\tp \myb_i,
\]
%
then 
\begin{align}\label{eq:rho}
\rho > \max\left\{0, \max_{j \in\left\{2,\ldots, |\mathcal{I}| \right\} } \frac{(\tx_{i}^{(1)} - \tx_{i}^{(j)})^\tp (\tx_{i}^{(1)} -\bx_i) +
\ridgepen (\weighti_{i}^{(1)})^{2}}{\weighti_{i}^{(j)} - \weighti_{i}^{(1)}} \right\}.
\end{align}
\end{lemma}


\noindent
Note that our definition of pairwise dissimilarities satisfies the requirements of the
lemma. The proof uses directional derivatives and the convexity of the objective function. 
In effect, Lemma~\ref{proof_wssr_nec} states that there are cases in which
the nearest-neighbour approximation is optimal for all $\rho > 0$. This occurs
when it is possible to define a hyperplane that contains $\tx_{i}^{(1)}$, the nearest neighbour of $\bx_{i}$, and separates the column vectors in $\Xmi$ from $\bx_{i}$.
In all other cases, there exists a positive value of $\rho$ such that $\myb_i^\star$
has cardinality greater than one.

We close this section by outlining
a simple proximal gradient descent algorithm~\citep{parikh2014proximal}
that is faster for large instances of the problem than standard QP solvers.
To this end, we first we express~\eqref{eq_wssr_qp} as an unconstrained
minimisation problem through the use of an indicator function,
\begin{equation*}\label{eq_unconstrained} 
\min_{\myb_i} f(\myb_i)+ \indi_{\Delta^{|\mathcal{I}|}} (\myb_i), 
\end{equation*}
%
where $\indi_{\Delta^{|\mathcal{I}|}}(\myb_i)$ is zero for $\myb_i \in
\Delta^{|\mathcal{I}|}$ and infinity otherwise.
At each iteration, proximal gradient descent updates $\myb_i^t$ by projecting
onto the unit simplex a step of the standard gradient descent,
\begin{equation*}\label{eq:ProxGrad1}
\myb_i^{\iter+1} = \argmin_{\myb \in \Delta^{|\mathcal{I}|}} \frac{1}{2}\| \myb
- \myb_i^{\iter} + \eta^{\iter} \nabla f(\myb_i^{\iter}) \|_2^2,
\end{equation*}
where $\eta^{\iter}$ is the step size at iteration $\iter$. 
Projecting onto $\Delta^{|\mathcal{I}|}$ can be achieved via a simple algorithm
with complexity $\mathcal{O}\left(|\mathcal{I}|
\log\left(|\mathcal{I}|\right)\right)$~\citep{wang2013projection}.

\section{Active learning and constrained clustering}\label{sec_wssr_si}

In this section, we describe the process of identifying informative points to
query (active learning), and then updating the clustering model to accommodate
the most recent labels (constrained clustering). The constrained clustering
algorithm described in this section would also be used if a subset of labelled
observations was available at the start of the learning process.
%
%
%




We adopt the active learning strategy of~\cite{peng2019subspace},
that queries the points whose label information is expected to induce the largest reduction
in the reconstruction error function (defined in~\eqref{eq:KSCobjective}).
Let $\U \subset \{1,\ldots,N\}$ denote the set of indices of the unlabelled points,
and $\L \subset \{1,\ldots,N\}$ denote the set of indices of labelled points.
Furthermore, let $\{\lab_{i}\}_{i =1}^N$ denote 
the cluster assignment of each point, and $\{l_i\}_{i \in \L}$ the class labels for the labelled points.
%
%
%
%
%
%
%
%
To quantify the expected reduction in reconstruction error after obtaining the label of $\myx_i$,
with $i \in \U$, we estimate two quantities.
The first is the decrease in reconstruction error that will result if $\myx_i$ is
removed from cluster $\lab_{i}$. This is measured by the function
$U_{1}(\myx_{i},V_{\lab_{i}})$, where $V_{\lab_i}$ is a matrix containing
a basis for cluster $\lab_i$. The second
is the increase in reconstruction
error due to the addition of $\myx_i$ to a different cluster $\lab_{i}^{\prime}$.
This is measured by the function $U_{2}(\myx_i,V_{\lab_i'})$. To estimate $U_2$
we assume that $\lab_i'$ is the cluster that is the second
nearest to $\myx_i$. This 
assumption is not guaranteed to hold but it is valid in the vast majority of cases.
%
%
According to~\cite{peng2019subspace} the most informative point to query is,
\begin{equation}\label{eq_al_util}
\myx_i^{\star}= \argmax_{i \in \U}
\left\{U_{1}(\myx_{i},V_{\lab_i})-U_2(\myx_{i},V_{\lab_{i}'}) \right\},
\end{equation}
%
%
%

The difficulty in calculating $U_{1}(\myx_{i},V_{\lab_i})$ and $U_2(\myx_{i},V_{\lab_{i}'})$
is that one needs to account for the fact that a change in the cluster assignment
of $\myx_i$ affects
%
%
$V_{\lab_i}$ and $V_{\lab_i'}$.
Recall that (irrespective of the choice of the clustering algorithm), the basis $V_k$ for each cluster (subspace) $k$  is computed by
performing Principal Component Analysis (PCA) on the set of points assigned to this
cluster. 
The advantage of the approach by~\cite{peng2019subspace} is that this
is recognised, and a computationally efficient method to approximate $U_1$ and $U_2$ is proposed.
%
%
%
%
%
%
%
In particular, using perturbation results for PCA~\citep{critchley1985influence}, a first-order approximation of the change in $V_{\lab_i}$ and $V_{\lab_i'}$ is
computed at a cost of $\mathcal{O}(\NumFea)$ (compared to
the cost of PCA which is $\mathcal{O}(\min\{N_k P^2, N_k^2 P\})$,
where $N_k$ is the number of points in cluster~$k$).
%

%
Once the labels of the queried points are obtained we proceed to
the constrained clustering stage in which
%
%
we update the cluster assignment to accommodate the new information.
%
%
The first step in our approach modifies pairwise
dissimilarities between points in a manner similar to the work of \cite{li2017structured}.
%
%
%
Specifically, for each $\myx_i \in \X$ we update all pairwise dissimilarities 
$\weighti_{ij} \in \weight_\I$ according to,
\begin{equation}\label{eq_newD}
\weighti_{ij} = \left\{ \begin{array}{ll}
\frac{\|\myx_{i}\|_2 \|\myx_{j}\|_2}{\myx_i^{\tp} \myx_{j}} e^{1 - 2 \cdot \indi(l_i=l_j)}
+ \alpha \indi(l_i \neq l_j),  & \;\; \text{if} \;\; i,j \in \L, \\
%
%
\frac{\|\myx_{i}\|_2 \|\myx_{j}\|_2}{\myx_i^{\tp} \myx_{j}} + \alpha \indi(c_i \neq c_j), &
\;\; \text{otherwise}.
\end{array} \right.
\end{equation}
The first fraction is the dissimilarity measure in the absence of any label information
as defined in~\eqref{eq_dij}.
%
%
If the labels of both $\myx_i$ and $\myx_j$ are known and they are different then the dissimilarity
is first scaled by $e$ and a constant $\alpha \in [0,1]$ is added.
If $l_i = l_j$, then the original dissimilarity is scaled by $e^{-1}$.
%
%
%
%
If the label of either $\myx_i$ or $\myx_j$ is unknown then no scaling is applied, but
if in the previous step the two points were assigned to different clusters then
their dissimilarity is increased by $\alpha$.
The term $\alpha$ quantifies the confidence 
of the algorithm in the previous cluster assignment. 
A simple and effective heuristic is to assign $\alpha$ equal to the proportion
of labelled data.

After updating pairwise dissimilarities through~\eqref{eq_newD}
we update the coefficient vectors for each point
by solving the problem in~\eqref{eq_wssr_qp}.
The resulting affinity matrix is the input to 
the normalised spectral clustering of~\cite{ng2002spectral}.
%
%
This cluster assignment is not guaranteed to satisfy all the constraints.
%
%
%
%
%
To ensure constraint satisfaction we use this clustering as the initialisation point for
the $K$-Subspace Clustering with Constraints (KSCC)
algorithm \citep{peng2019subspace}.
KSCC is an iterative algorithm to optimise the following
objective function,
%
%
%
\begin{equation}\label{eq_kscc}
\min_{V_1,\ldots, V_K} \frac{1}{2} \left\{ \sum_{i \in \U}  \min_{k=1,\ldots,K} \|\myx_i - V_k V_k^\tp \myx_i\|^2_2
+ \min_{ \substack{ P \in \mathcal{P}(K)\\ n \in 1,\ldots,K!}} \; \sum_{k=1}^K \; \sum_{ \substack{j \in \L : \\l_j = k}}
\|\myx_j  - V_{P_{k}} V_{P_{k}}^\tp \myx_j\|_2^2,
\right\},
\end{equation}
where $P$ denotes a permutation of the indices $\{1,\ldots,K\}$,
and $\mathcal{P(K)}$ is the set of all such permutations.
The first term is the reconstruction error for the unlabelled points. The second
term quantifies the reconstruction error for the labelled points. To this end we have
to identify the appropriate mapping between class labels and cluster labels. Hence
we consider all possible mappings (permutations of labels one to $K$) and select
the one producing the lowest reconstruction error. The inner sum in the second term ensures that 
all points of the same class are assigned to a unique cluster.  Thus
all the constraints are satisfied at each iteration.
%
%
KSCC monotonically reduces the value of the objective function.
Therefore it converges to the local minimum of \eqref{eq_kscc}
whose region of attraction contains the WSSR cluster assignment.
%
%
The computational complexity of KSCC is the same as KSC, which is $\mathcal{O}(\min\{N_k P^2, N_k^2 P\})$.
%
%
We summarise the whole active learning and constrained clustering framework in procedural form in Algorithm~\ref{algo_all}. We refer to this constrained version of WSSR as WSSR+. 
\begin{algorithm}[htbp!]
	\caption{Active Learning and Constrained Clustering with WSSR}
	\DontPrintSemicolon
	\SetAlgoLined
	\SetKwInOut{Input}{Input}
	\SetKwInOut{Output}{Output}
	\Input{WSSR-related parameters; 
		Sets of `must-link' and `cannot-link' constraints: $\mathcal{S}_{M}, \mathcal{S}_{C}$; 
		Penalty parameter: $\alpha$;
		Number of points to query in each iteration: $b$
	}
	\% \emph{Active learning}\\ 
	- Query the $b$ most informative point according to~\eqref{eq_al_util}\\
	\% \emph{Constraint incorporation}\\
	\textbf{For} $\myx\in\mathcal{X}$:\\
	1. Compute the updated weight vector $\weight^{\star}$ according to~\eqref{eq_newD}\;
	2. Normalise and stretch each column vector in $X$ \;
	3. Solve the WSSR problem in~\eqref{eq_wssr_en} 
	to obtain the coefficient vector $\bm{\beta}$\;
	\textbf{End}
	- Combine all $\bm{\beta}$s to obtain the coefficient matrix $B\in\mathbb{R}^{N\times N}$\;
	- Apply normalised cut spectral clustering \citep{ng2002spectral} to the data affinity matrix $A=\frac{1}{2}\left(|B|+|B|^{\tp} \right)$\\
	\% \emph{Constraint satisfaction}\\
	- Enforce the constraint information using KSCC~\citep{peng2019subspace} and obtain the updated cluster labels 
	\label{algo_all}
\end{algorithm}

\section{Experiments on Synthetic Data}\label{sec_wssr_syn}
In this section, we conduct experiments on synthetic data to evaluate the performance of WSSR under various subspace settings. 
We compare to the following state-of-the-art spectral-based subspace clustering methods: SSC \citep{elhamifar2013sparse}, S3C \citep{li2015structured}, ASSC \citep{li2018geometric}, SSC-OMP \citep{you2016scalable}, LSR \citep{lu2012robust}, and SMR \citep{hu2014smooth}. Performance results of various methods are compared in terms of varying angles between subspaces, varying noise levels, and varying subspace dimensions.~\footnote{The code of our proposed method is available at:~\url{https://github.com/hankuipeng/WSSR}}

\subsection{Varying Angles between Subspaces} 
In this set of experiments, we generate data from two one-dimensional subspaces embedded in a three-dimensional space. Each cluster contains 200 data points drawn from one of the subspaces. In addition, additive Gaussian noise with standard deviation $\sigma=0.01$ is added to the data uniformly. We vary the angles between the two subspaces $\theta$ to be between 10 and 60 degrees, and evaluate the performance of various algorithms under each setting. 
The default settings are adopted for all subspace clustering algorithms that we provide comparisons to. Let $k$ denote the maximum number of points being considered in the sparse representation. Both SSC-OMP and SMR have $k=10$ in their default parameter settings. We adopt the same setting and set $\rho=0.01$.
Performance results as evaluated by clustering accuracy are reported in Table~\ref{tab_va_res}. 

It can be seen that WSSR achieves the best performance across all settings, and its performance steadily improves with the increase of the angles between subspaces. Similar performance improvement with the increase of $\theta$ can also be observed in SSR, as well as other algorithms. However, the performance of SSR is significantly worse than that of WSSR when the angles are small. 
SSC and S3C achieve strong performance across all scenarios as well. It is worth noting that the performance of ASSC is noticeably worse than other algorithms, which could be explained by the fact that all clusters come from linear subspaces. 
The performance of SMR is a close second to WSSR in five out of six scenarios. This could be attributed to its affinity to WSSR, as SMR applies the Frobenius norm on the error matrix and it also makes use of $k$ nearest neighbours.
\begin{table}[h!]
	\begin{center}
		\begin{tabular}{ c c c c c c c }
			\hline
			&$\theta=$ 10& $\theta=$ 20&$\theta=$ 30&$\theta=$ 40 &$\theta=$ 50&$\theta=$ 60\\
			\hline
			WSSR &\textbf{0.978}&\textbf{0.973}&\textbf{0.993}&\textbf{0.993}&\textbf{0.990}&\textbf{0.993}\\
			\hline 
			SSR &0.568&0.865&0.528&0.905&$\underline{0.985}$&\textbf{0.993}\\
			\hline
			SSC&0.943&0.815&\underline{0.990}&0.950&$\underline{0.985}$&\textbf{0.993}\\
			\hline
			S3C&0.963&\textbf{0.973}&\underline{0.990}&0.970&0.983&\textbf{0.993}\\
			\hline
			ASSC&0.520&0.570&0.570&0.568&0.510&0.555\\
			\hline
			SSC-OMP&0.863&0.893&0.848&0.823&0.528&0.813\\
			\hline
			LSR&0.898&0.878&0.900&0.873&0.940&0.930\\ 
			\hline
			SMR&\underline{0.968}&\underline{0.960}&\underline{0.990}&\underline{0.975}&0.978&\underline{0.988}\\
			\hline
		\end{tabular}
	\end{center}
	\caption{Accuracy of various subspace clustering algorithms on synthetic data with varying angles between subspaces.}
	\label{tab_va_res}
\end{table}

\subsection{Varying Noise Levels}
Next, we explore the effect of various noise levels on cluster performance. Again we generate data from two subspaces, each containing 200 data points. The angle between two subspaces is set to be 60 degrees, so that the angle between subspaces does not play a big role in determining the cluster performance. One of the subspaces is one-dimensional, and the other is two-dimensional. This difference to the previous set of experiments is to increase the intersection between the two subspaces as the noise level increases. Additive Gaussian noise with zero mean and standard deviation $\sigma$ is added to the data uniformly, in which $\sigma$ ranges from 0.0 to 0.5. The parameter settings for all algorithms remain the same as before, and the performance results are reported in Table \ref{tab_vn_res}. 


Firstly, the clustering accuracy of almost all algorithms decreases with the increase of noise levels. Secondly, most algorithms have close to or exactly perfect clustering accuracy in the noise-free scenario. However the performance of most algorithms degrades rapidly with the increase of noise level, whereas the performance of WSSR stays competitive.
Very poor performance from ASSC can be observed across all noise levels, for reasons explained in the previous set of experiments. 
It is worth noting that all SSC-based methods (SSC, S3C, ASSC, SSC-OMP) yield poor performance in the presence of varying levels of noise. 
In comparison, the methods (WSSR, SSR, LSR, SMR) that use the Frobenius norm on the error matrix or the $L_2$-norm on the reconstruction error have favourable performance. However, none of them (apart from WSSR) have a sparsity inducing term in the objective function. 
\begin{table}[h!]
	\begin{center}
		\begin{tabular}{ c c c c c c c }
			\hline
			&$\sigma=$ 0.0& $\sigma=$ 0.1&$\sigma=$ 0.2&$\sigma=$ 0.3&$\sigma=$ 0.4&$\sigma=$ 0.5\\
			\hline
			WSSR&\textbf{1.000}&\textbf{0.970}&\textbf{0.945}&\textbf{0.883}&\textbf{0.815}&\textbf{0.745}\\
			\hline
			SSR&\textbf{1.000}&0.940&0.848&0.805&0.780&0.725\\
			\hline
			SSC&\textbf{1.000}&0.633&0.503&0.513&0.508&0.555\\
			\hline
			S3C&\underline{0.980}&0.685&0.575&0.608&0.523&0.593\\
			\hline
			ASSC&0.605&0.530&0.553&0.558&0.503&0.543\\
			\hline
			SSC-OMP&\textbf{1.000}&0.730&0.575&0.510&0.518&0.530\\
			\hline
			LSR&\textbf{1.000}&\underline{0.943}&\underline{0.900}&\underline{0.838}&\underline{0.788}&\underline{0.735}\\ 
			\hline
			SMR&\underline{0.980}&0.935&0.868&0.810&0.775&0.705\\
			\hline
		\end{tabular}
	\end{center}
	\caption{Accuracy of various subspace clustering algorithms on synthetic data with varying noise levels.}
	\label{tab_vn_res}
\end{table}

Next, we further investigate the reason behind the slightly less than perfect performance of S3C and SMR, and the poor performance of ASSC in the noise-free scenario. Shown in Figure \ref{fig_vn_w} are the affinity matrices for WSSR, S3C, SMR, and ASSC. 
It is clear to see that the entries in the affinity matrix of WSSR are sparse yet the block-diagonal structure is very clear. 
The affinity matrix of S3C also exhibits a block-diagonal structure and is very sparse. However, the majority of the non-zero entries are concentrated on a few key data points. Such an unbalanced affinity matrix can lead to undesirable spectral clustering performance.  
The affinity matrix of both SMR and ASSC are very dense. However the block diagonal structure can still be easily detected in the affinity matrix of SMR, whereas the entries in ASSC are more evenly spread out.
Therefore, it is not surprising that the performance of ASSC is much worse than that of the other methods. 
\begin{figure}[h!]
	\centering
	\begin{subfigure}[b]{0.24\textwidth}
		\centering
		\includegraphics[width=\textwidth, trim={1cm, 0.5cm, 1cm, 0.5cm}]{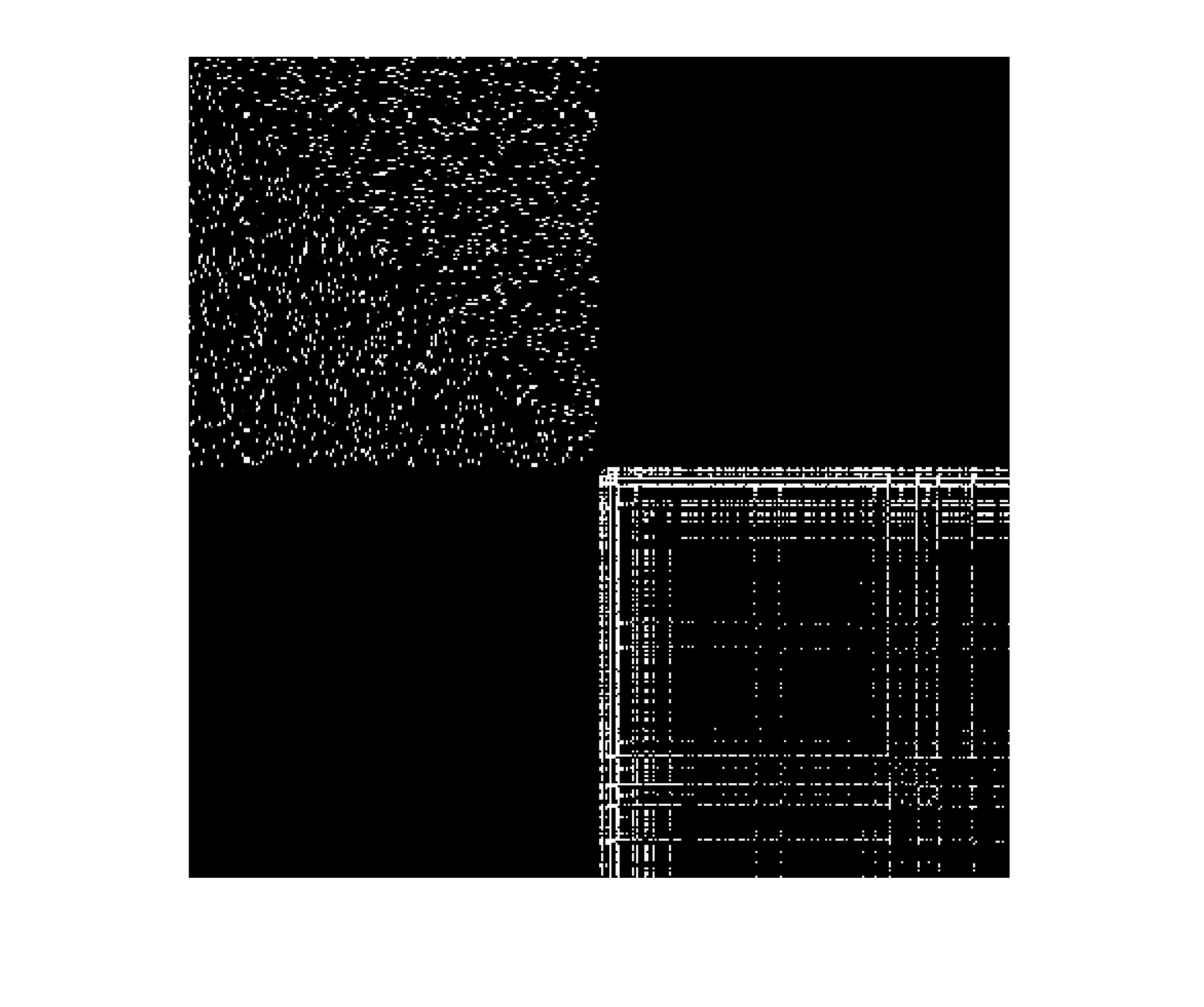}
		\caption{WSSR.}
	\end{subfigure}
	\begin{subfigure}[b]{0.24\textwidth}
		\centering
		\includegraphics[width=\textwidth, trim={1cm, 0.5cm, 1cm, 0.5cm}]{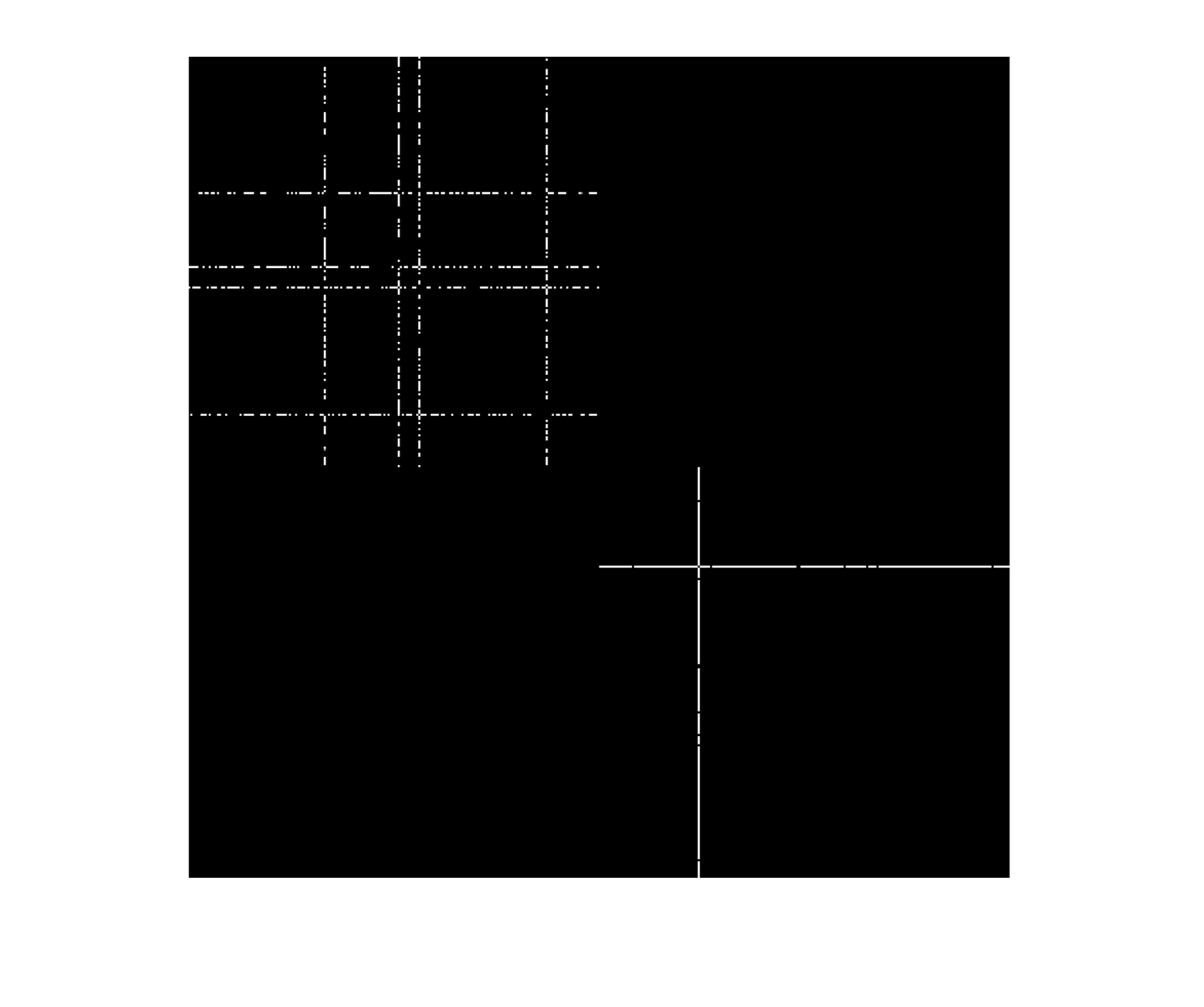}
		\caption{S3C.}
	\end{subfigure}
	\begin{subfigure}[b]{0.24\textwidth}
		\centering
		\includegraphics[width=\textwidth, trim={1cm, 0.5cm, 1cm, 0.5cm}]{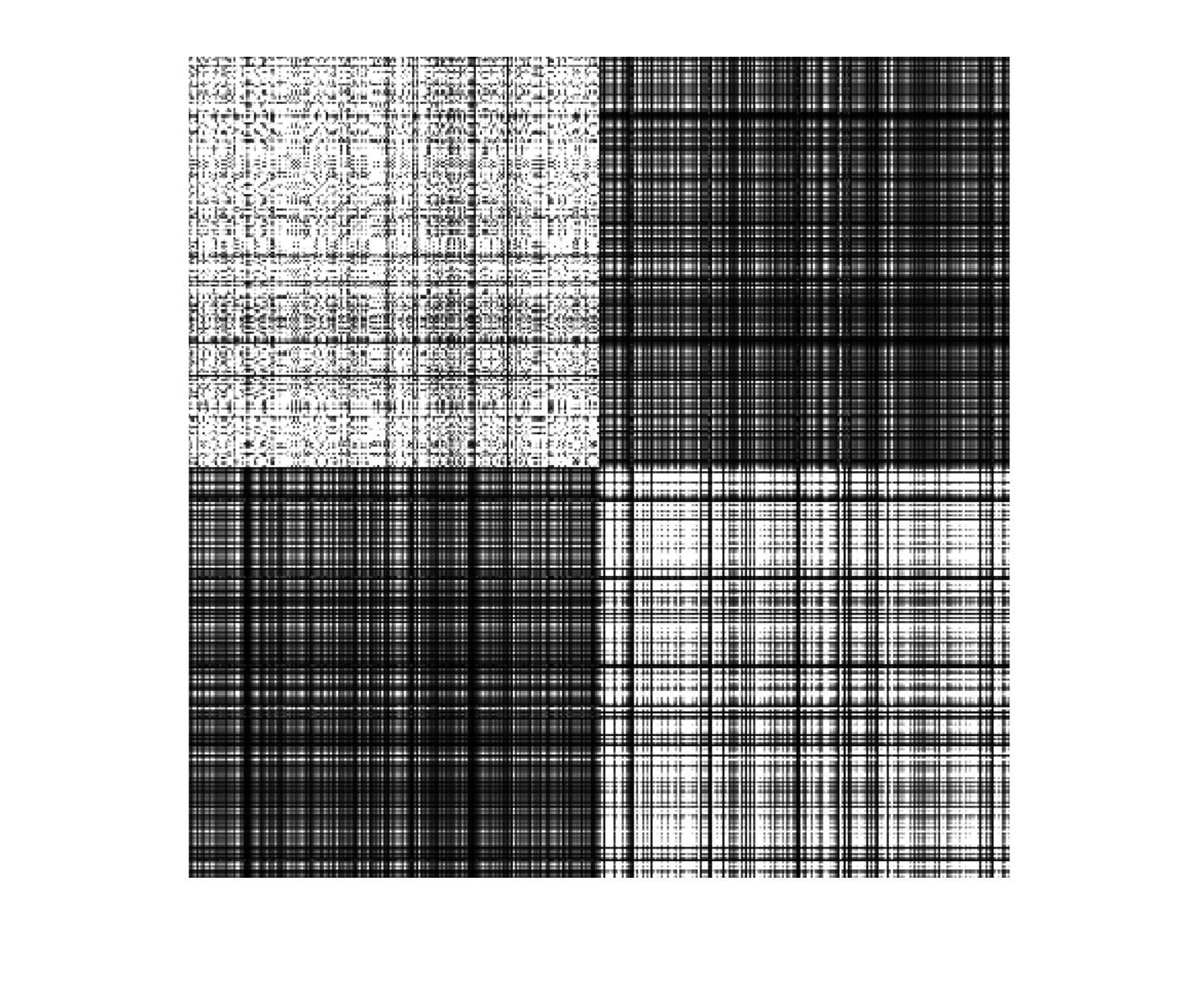}
		\caption{SMR.}
	\end{subfigure}
	\begin{subfigure}[b]{0.24\textwidth}
		\centering
		\includegraphics[width=\textwidth, trim={1cm, 0.5cm, 1cm, 0.5cm}]{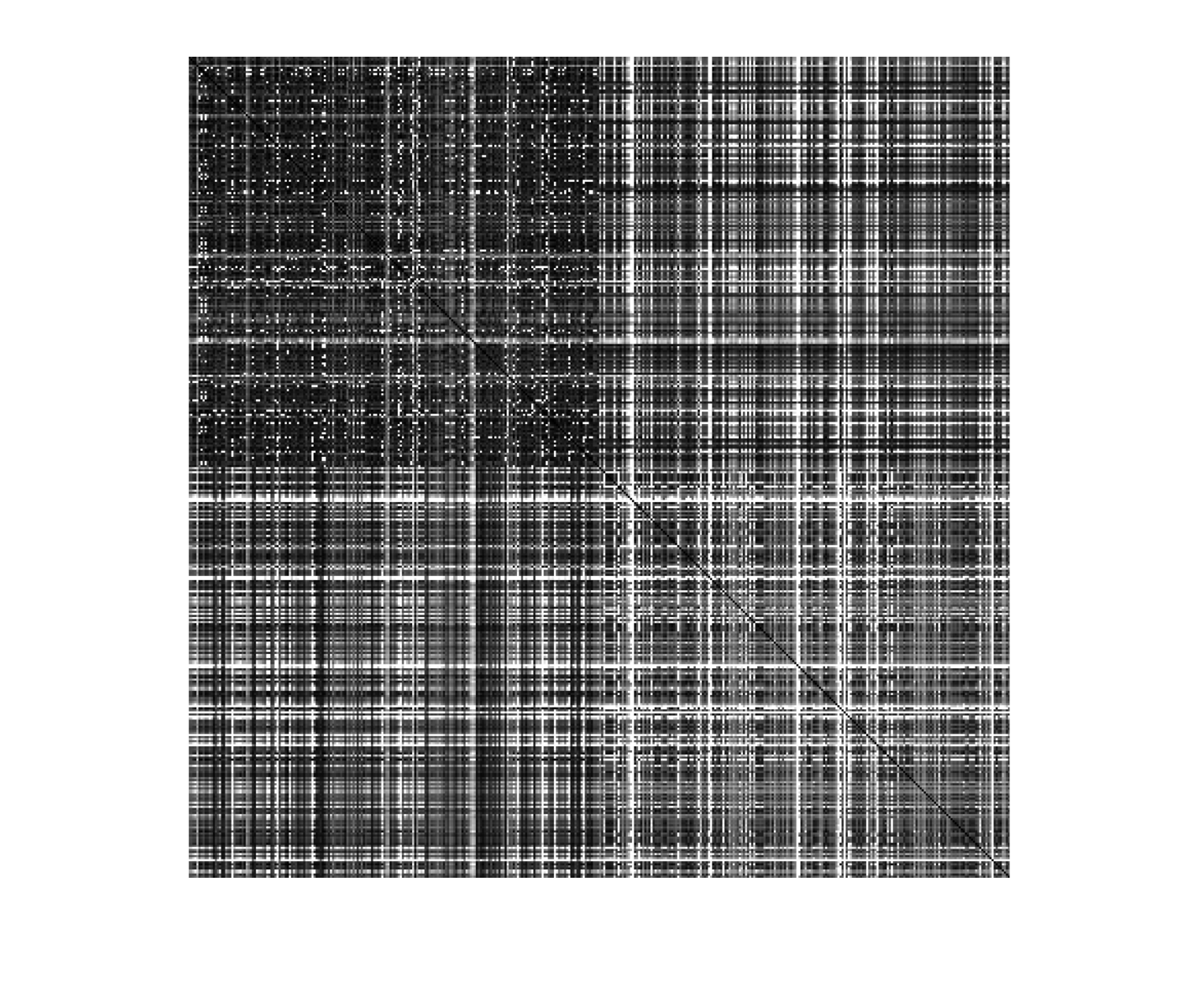}
		\caption{ASSC.}
	\end{subfigure}
	\caption{Visualisation of the affinity matrix in the noise-free scenario. 
	Zeroes in the affinity matrix are shown in black colour.}
	\label{fig_vn_w}
\end{figure}

\subsection{Varying Subspace Dimensions}\label{sec_synexp_vq}
Another aspect we would like to investigate is how clustering performance changes with the increase of subspace dimensions $\subdim_{k}$ under a fixed ambient space dimension $P$. Intuitively, one would expect the performance to get worse as $\subdim_{k}$ increases. This is because the intersection among clusters is likely to increase as the subspace dimension grows. In this set of experiments, we fix the ambient space dimension to be 20 and allow the subspace dimension to vary between 2 and 16. The data are generated from 4 subspaces, and all subspaces have equal subspace dimension. Additive noise with $\sigma=0.01$ is added to the data uniformly. 
We set the neighbourhood size $k$ in WSSR to be 50, and keep $\rho=0.01$. The default parameter settings are adopted for all other methods. The performance results of various methods are reported in Table~\ref{tab_vq_res}.
\begin{table}[ht!]
	\begin{center}
		\begin{tabular}{ c c c c c c c c c }
			\hline
			&$\subdim_{k}=$ 2& $\subdim_{k}=$ 4&$\subdim_{k}=$ 6&$\subdim_{k}=$ 8&$\subdim_{k}=$ 10&$\subdim_{k}=$ 12&$\subdim_{k}=$ 14&$\subdim_{k}=$ 16\\
			\hline
			WSSR&\textbf{1.000}&\textbf{1.000}&\textbf{1.000}&\textbf{1.000}&\textbf{1.000}&\textbf{1.000}&\textbf{0.991}&\textbf{0.874}\\
			\hline
			SSR&$\underline{0.999}$&\textbf{1.000}&\textbf{1.000}&$\underline{0.998}$&$\underline{0.999}$&0.951&0.383&0.328\\
			\hline
			SSC&\textbf{1.000}&\textbf{1.000}&\textbf{1.000}&\textbf{1.000}&\textbf{1.000}&\textbf{1.000}&\underline{0.964}&0.728\\
			\hline
			S3C &\underline{0.999}&\textbf{1.000}&\textbf{1.000}&\textbf{1.000}&\textbf{1.000}&\textbf{1.000}&\textbf{0.991}&\underline{0.809}\\
			\hline
			ASSC&0.845&0.998&\textbf{1.000}&\textbf{1.000}&\textbf{1.000}&\underline{0.994}&0.954&0.599\\
			\hline
			SSC-OMP&0.414&\textbf{1.000}&\underline{0.999}&0.993&0.966&0.913&0.475&0.321\\
			\hline
			LSR&0.861&\underline{0.999}&\textbf{1.000}&\textbf{1.000}&0.995&0.986&0.936&0.518\\
			\hline
			SMR&\underline{0.999}&0.998&0.968&0.981&0.823&0.469&0.335&0.323\\
			\hline
		\end{tabular}
	\end{center}
	\caption{Accuracy of various subspace clustering algorithms on synthetic data with varying subspace dimensions.}
	\label{tab_vq_res}
\end{table}

We observe that WSSR achieves the best performance across all settings, though the accuracy becomes less than perfect when $\subdim_{k}$ is greater than 12. 
This is because the increase of subspace dimensions likely increases the intersection between subspaces hence induces more noise. 
When we move our focus to the performance of SSC-based methods, we see that both SSC and S3C also have near perfect clustering accuracy for $\subdim_{k}$ up to 12. Then the results get slightly worse for higher values of $\subdim_{k}$. For $\subdim_{k}$ less than 16, the results of ASSC are much better than in the previous set of varying noise experiments. This shows that ASSC has the subspace recovery ability when the noise level is small. 
SSC-OMP has worse accuracy than its base method SSC, which could be attributed to its default neighbourhood size. SSR, LSR and SMR maintain good performance throughout when $P_{k}$ is relatively small, however their performance degrades sharply as $P_{k}$ increases.

\section{Experiments on Real Data}\label{sec_wssr_real}

In this section we use real datasets to assess the performance of WSSR with and without side
information. As mentioned at the end of Section~\ref{sec_wssr_si}, we refer to
WSSR with side information as WSSR+. We evaluate the following aspects of
our proposed framework: (1) How does WSSR
compare with other state-of-the-art subspace clustering methods? (2) How does
the performance of WSSR+ compare with other constrained clustering methods? (3)
What are the additional benefits of the active learning component?

\subsection{Experiments on MNIST}

In this section, we conduct experiments comparing WSSR with other
state-of-the-art subspace clustering algorithms on the MNIST handwritten digits
data \citep{lecun1998gradient}. This database contains greyscale images of
handwritten digits numbered from 0 to 9, and has been widely used in the machine
learning literature to benchmark the performance of supervised and unsupervised
learning methods.
In the context of subspace clustering 
\cite{you2016scalable} used the MNIST data set to demonstrate the effectiveness of SSC-OMP.
%
%

We compare the performance of WSSR to that of SSC-OMP, 
SSC \citep{elhamifar2013sparse}, ASSC \citep{li2018geometric}, S3C
\citep{li2015structured}, LSR \citep{lu2012robust}, SMR \citep{hu2014smooth},
and FGNSC \citep{yang2019subspace}. 
We use the default parameter settings for SSC, ASSC, and SSC-OMP. For
LSR we use the standard version of the algorithm that includes
the constraint that the diagonal entries of the coefficient matrix must be
zero. 
%
The default parameter setting
for SMR uses $k=4$ for the $k$-nearest neighbour graph. We adopt the default
setting of FGNSC, which uses SMR as the base algorithm to obtain the initial
affinity matrix.  In WSSR, we set $k=10$ and $\rho=0.01$. That is, we consider a
10-nearest neighbourhood for each data point, which is the same as the default
setting in SSC-OMP. 


The original MNIST data set contains 60,000 points in $P=3472$ dimensions organised
in 10 clusters (digits 0
to 9). We conduct two sets of
experiments on this data. The first set of experiments investigates the effect of
the number of clusters $\ncluster$ on the performance of various
algorithms. We randomly select $K\in\left\{2, 3, 5, 8, 10 \right\}$ clusters
out of the 10 digits. Each cluster contains 100 randomly sampled points, and
the full dataset is then projected onto 200 dimensions using PCA. The second
set of experiments explores the effect of the total number of data points $N$ on the
performance, following the experimental design proposed by \cite{you2016scalable}.
%
%
Specifically, we randomly select $N_{k}\in\left\{50, 100, 200, 400,
600\right\}$ points from each cluster out of all 10 clusters. The full data are
then projected onto 500 dimensions using PCA. 
For the each choice of $K$ and $N_k$ we
randomly sample 20 data sets and apply all methods
on these 20 sets.
Table~\ref{tab_mnist_vk} reports the 
median and standard deviation of clustering accuracy
of each algorithm. 

\begin{table}[h!]
	\begin{center}
		\begin{tabular}{l| cc cc cc cc cc}
			\hline
			&\multicolumn{2}{c}{$K=2$} &\multicolumn{2}{c}{$K=3$} &\multicolumn{2}{c}{$K=5$} &\multicolumn{2}{c}{$K=8$} &\multicolumn{2}{c}{$K=10$}\\
			&Med&Std&Med&Std&Med&Std&Med&Std&Med&Std\\
			\hline
			WSSR&\textbf{1.00}&0.01&\textbf{1.00}&0.01&\textbf{0.99}&0.01&\textbf{0.98}&0.01&\textbf{0.98}&0.02\\
			\hline 
			SSR&0.95&0.09&0.89&0.11&0.68&0.11&0.64&0.06&0.60&0.05\\
			\hline
			SSC &\underline{0.99}&0.03&0.91&0.05&0.80&0.08&0.78&0.03&0.81&0.02\\ 
			\hline
			S3C&\underline{0.99}&0.02&$\underline{0.97}$&0.08&0.79&0.08&0.82&0.04&0.81&0.03\\
			\hline
			ASSC &\underline{0.99}&0.01&$\underline{0.97}$&0.02&\underline{0.95}&0.05&0.87&0.06&0.84&0.04\\
			\hline
			SSC-OMP &0.98&0.01&$\underline{0.97}$&0.02&\underline{0.95}&0.04&\underline{0.90}&0.04&\underline{0.86}&0.03\\ 
			\hline
			LSR &0.98&0.11&0.68&0.11&0.86&0.07&0.82&0.04&0.78&0.02\\
			\hline
			SMR&\underline{0.99}&0.04&0.98&0.02&\underline{0.95}&0.06&0.86&0.05&0.83&0.03\\
			\hline
			FGNSC &\underline{0.99}&0.06&0.98&0.03&0.83&0.09&0.84&0.04&0.85&0.04\\
			\hline
		\end{tabular}
	\end{center}
	\caption{Median clustering accuracy along with the standard deviations on the MNIST handwritten digits data across 20 replications with varying number of clusters.}
	\label{tab_mnist_vk}
\end{table}

%
It can be seen that WSSR achieves the best performance across all settings. In particular, it achieves perfect clustering accuracy for both $K=2$ and 3. It is also worth noting that the performance variability is relatively small as compared to other methods. 
For the competing algorithms, we see that a few algorithms have excellent performance when $K$ is small. However the performance degrades in general with the increase of $K$. This is especially the case with SSR, which performs poorly even when $K$=5. 
Both SSC and S3C have very similar performance to each other, with S3C having slightly higher accuracy scores on two scenarios. 
The performance of ASSC is higher than both of the previous two, and has similar performance to that of SSC-OMP. However, SSC-OMP has an obvious advantage over the other SSC-based methods when $K$ is large. 
The median performance of LSR seems to be the worst out of all methods, and is the most variable especially when $K$ is small.
SMR and FGNSC have the same level of performance when $K=2$ and 3, which is not surprising given that the affinity matrix of FGNSC is adapted from that of SMR.
However, the performance of FGNSC is less stable for larger values of $K$. It has been observed in our experiments that the performance of FGNSC is highly sensitive to the parameter values.

The performance results for the second set of experiments with varying number of points per cluster are reported in Table~\ref{tab_mnist_vn}. 
WSSR is again the best performing algorithm across all settings, and indeed the most stable of all as well. 
It can be seen that SSC-OMP is the algorithm with the second best performance in most settings. The performance of SSC-OMP reported here is consistent with what is reported in \citep{you2016scalable}.
In general, the cluster performance improves with increasing number of points. Although this increase does not seem to provide an obvious performance boost for SSR.
At the same time, the performance for SSC, ASSC, and LSR also become more stable as $N$ increases, which is not the case for SMR.  
With that said, FGNSC provides performance improvement in addition to the performance of SMR in most scenarios. 
\begin{table}[h!]
	\begin{center}
		\begin{tabular}{l| cc cc cc cc cc}
			\hline
			&\multicolumn{2}{c}{$N=500$} &\multicolumn{2}{c}{$N=1000$} &\multicolumn{2}{c}{$N=2000$} &\multicolumn{2}{c}{$N=4000$} &\multicolumn{2}{c}{$N=6000$}\\
			&Med&Std&Med&Std&Med&Std&Med&Std&Med&Std\\
			\hline
			WSSR &\textbf{0.96}&0.03&\textbf{0.98}&0.00&\textbf{0.98}&0.00&\textbf{0.99}&0.00&\textbf{0.99}&0.00\\
			\hline 
			SSR &0.55&0.04&0.61&0.05&0.61&0.04&0.60&0.03&0.60&0.04\\ 
			\hline
			SSC &0.78&0.04&0.81&0.03&0.82&0.02&0.83&0.01&0.84&0.01\\ 
			\hline
			S3C &0.78&0.04&0.82&0.04&0.82&0.03&0.83&0.02&0.83&0.02\\
			\hline
			ASSC &0.82&0.04&0.85&0.04&0.83&0.02&0.82&0.01&0.83&0.01\\
			\hline
			SSC-OMP &\underline{0.83}&0.04&\underline{0.88}&0.03&\underline{0.91}&0.02&\underline{0.92}&0.03&0.92&0.04\\ 
			\hline
			LSR &0.66&0.04&0.74&0.03&0.78&0.02&0.79&0.01&0.80&0.01\\
			\hline
			SMR&0.76&0.03&0.82&0.05&0.88&0.04&0.86&0.04&0.92&0.04\\
			\hline
			FGNSC &0.77&0.03&0.85&0.03&0.87&0.04&0.87&0.04&\underline{0.97}&0.02 \\
			\hline
		\end{tabular}
	\end{center}
	\caption{Median clustering accuracy along with the standard deviations on the MNIST handwritten digits data across 20 replications with varying number of points per cluster.}
	\label{tab_mnist_vn}
\end{table}


All experiments were performed on a cloud computing machine with 5 CPU cores
and 15GB of RAM. For each of the two sets of experiments on the MNIST data set,
we present a comparison of the median computational times in log-scale for
different algorithms.  
%
%
%
Both figures indicate that S3C is the most computationally intensive method.
SSC and ASSC have similar computational times because they are based on the
same optimisation framework.  SSC-OMP and LSR are the most efficient out of all
methods. SSC-OMP is the SSC variant that is most suitable for large-scale
problems, while LSR has a closed-form solution.  The computational time for
WSSR is comparable to that of SSC and ASSC for small
problems, while the comparison becomes more favourable to WSSR when as $N$ and
$K$ increase. In particular from Fig.~\ref{fig_mnist_times_b} we see that for
larger instances ($N \geqslant 200$) WSSR is the third fastest algorithm,
while its running time is almost equal to that of LSR when $N=600$.
\begin{figure}[htbp!]
\centering
\begin{subfigure}{.49\textwidth}
	\includegraphics[width=1.1\textwidth, height=1.1\textwidth]{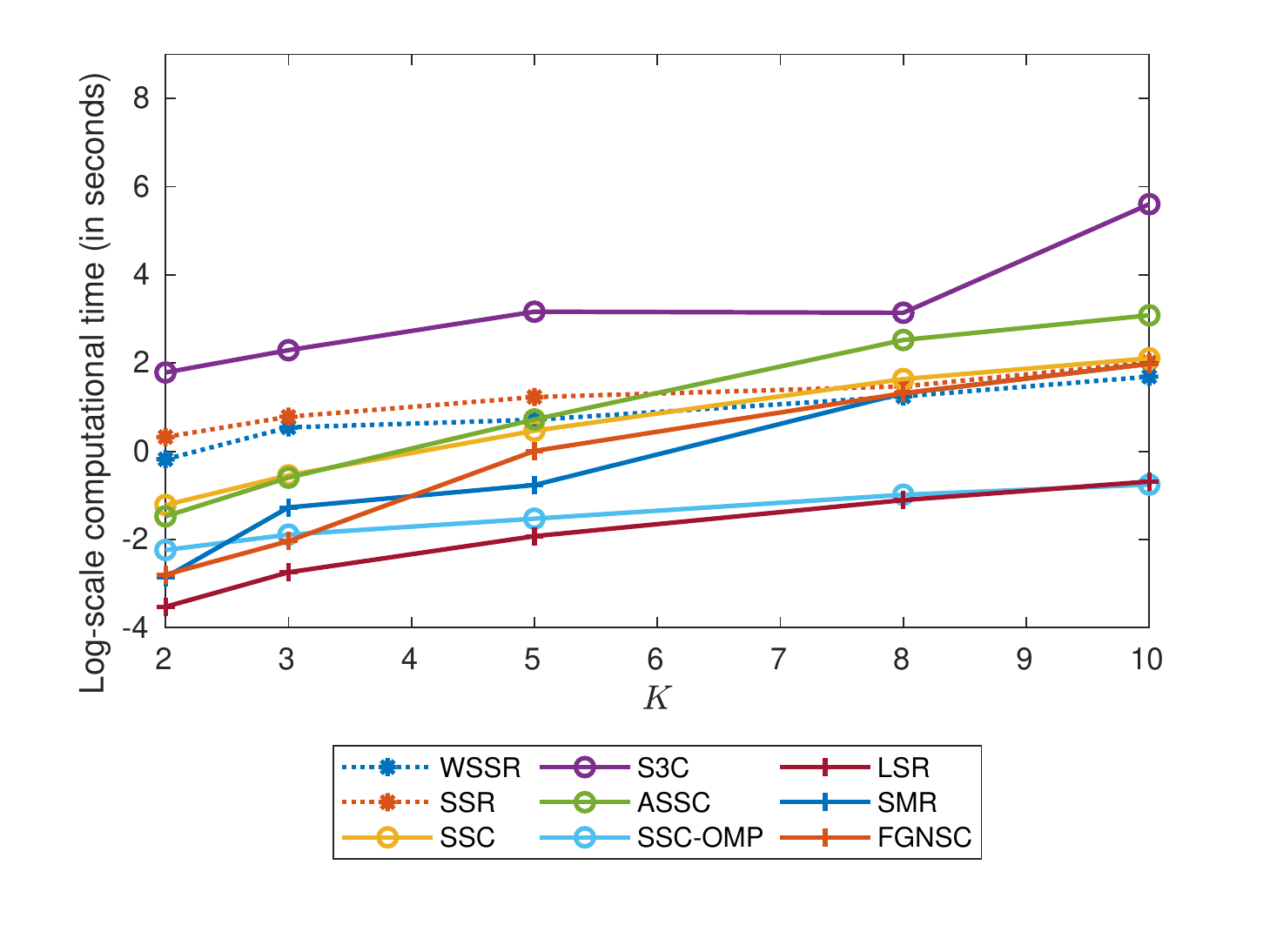}
	\caption{Varying number of clusters ($K$).}
	\label{fig_mnist_times_a}
\end{subfigure}
\begin{subfigure}{.49\textwidth}
	\includegraphics[width=1.1\textwidth, height=1.1\textwidth]{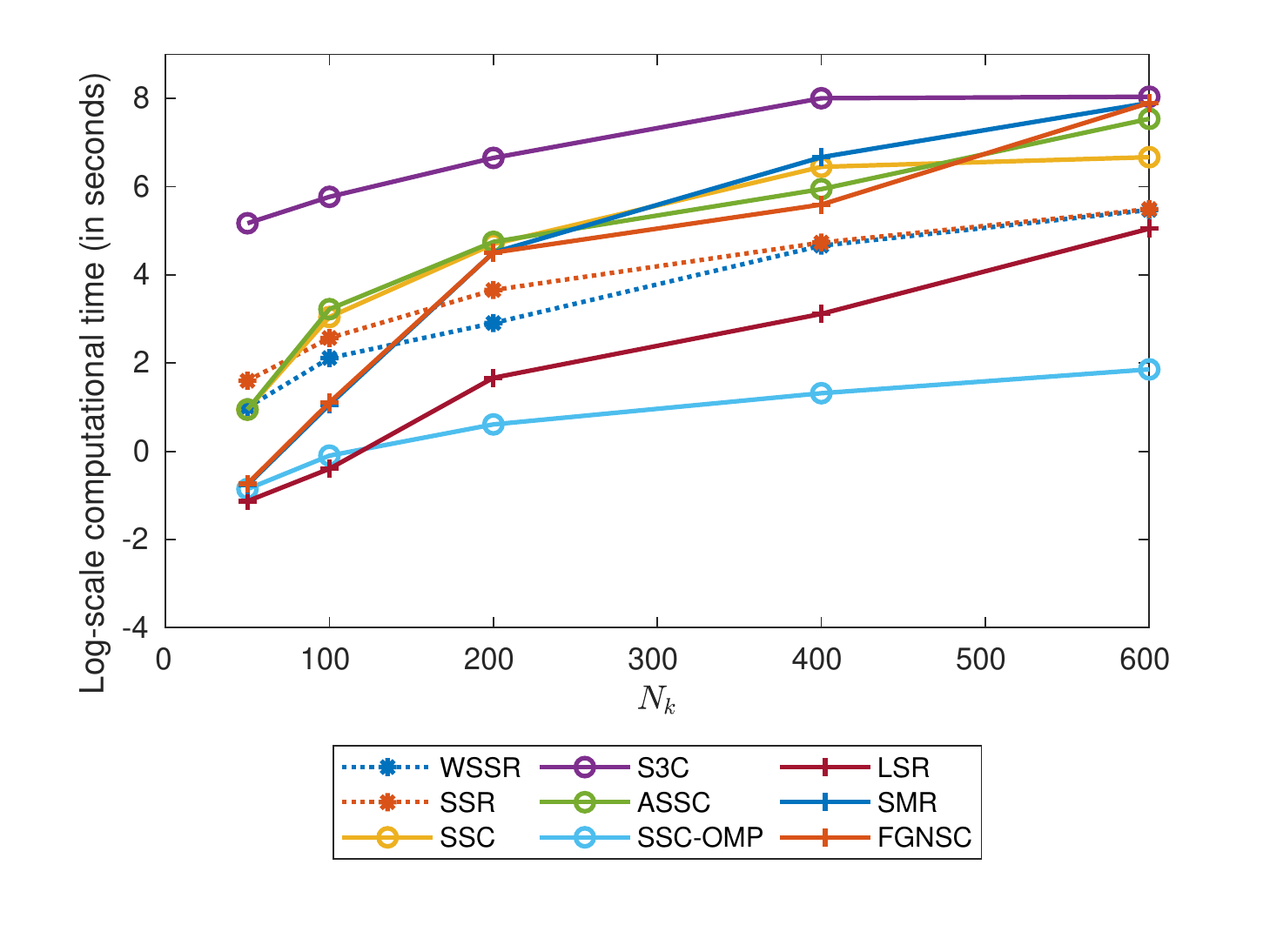}
	\caption{Varying number of points per cluster ($N_{k}$).}
	\label{fig_mnist_times_b}
\end{subfigure}
\caption{Median running times (in log-scale of seconds) of different algorithms on the MNIST handwritten digits data.}
\end{figure}

 

\subsection{Experiments on USPS}

In this subsection, we evaluate the performance of different subspace
clustering methods on the USPS digits data \citep{hull1994database}. USPS
is another widely used benchmark data set, which has been used to demonstrate the
effectiveness of subspace clustering methods \citep{hu2014smooth,
yang2019subspace}. USPS consists of 9298 images of handwritten digits that range from
0 to 9, and each image contains 16$\times$16 pixels. We follow the exact same
experimental settings as in \cite{hu2014smooth}, which uses the first 100
images from each digit.
\begin{table}[h!]
	\begin{center}
		\begin{tabular}{l| cc cc cc cc cc}
			\hline
			&\multicolumn{2}{c}{$K=2$} &\multicolumn{2}{c}{$K=3$} &\multicolumn{2}{c}{$K=5$} &\multicolumn{2}{c}{$K=8$} &\multicolumn{2}{c}{$K=10$}\\
			&Med&Std&Med&Std&Med&Std&Med&Std&Med&Std\\
			\hline
			WSSR&\textbf{1.00}&0.01&\textbf{0.99}&0.01&\textbf{0.98}&0.01&\textbf{0.97}&0.00&\textbf{0.97}&0.00\\
			\hline 
			SSR&0.96&0.10&0.95&0.08&0.80&0.09&0.76&0.06&0.69&0.01\\
			\hline
			SSC &0.96&0.03&0.85&0.12&0.63&0.14&0.60&0.05&0.51&0.00\\ 
			\hline
			S3C &0.97&0.02&0.83&0.13&0.69&0.09&0.62&0.05&0.51&0.01\\
			\hline
			ASSC&0.97&0.02&0.83&0.12&0.61&0.14&0.62&0.06&0.58&0.00\\
			\hline
			SSC-OMP&0.95&0.04&0.74&0.11&0.42&0.09&0.41&0.09&0.35&0.01\\ 
			\hline
			LSR&0.71&0.14&0.63&0.13&0.65&0.05&0.55&0.04&0.53&0.01\\
			\hline
			SMR&\underline{0.99}&0.04&\underline{0.98}&0.02&\underline{0.95}&0.06&\underline{0.86}&0.05&0.83&0.03\\
			\hline
			FGNSC &\underline{0.99}&0.02&\underline{0.98}&0.05&\underline{0.95}&0.06&0.83&0.05&\underline{0.85}&0.02\\
			\hline
		\end{tabular}
	\end{center}
	\caption{Median clustering accuracy along with the standard deviations on the USPS data across 20 replications.}
\end{table}

We investigate the performance of various algorithms under varying number of
clusters $K$. All experiments are conducted for 20 replications, and we report both
the median and standard deviation of the clustering accuracy. For $K$ from 2 to
8, we randomly sample data from $K$ digits.  
Therefore the variability
in the cluster performance comes from both the variability in the subset of the
data, and the variability of the corresponding algorithm. For $K=10$, we use
the same data set with 1000 images across all replications. In this case, the
standard deviation reflects only the variability of the algorithms. 

%
%
SSC, S3C, and ASSC exhibit similar performance for all values of $K$. 
On the USPS data set, the performance of these methods degrades much more as $K$ increases compared to the MNIST data set.
Performance variability is also higher for $K=2,3$ as evinced by the higher values of the reported standard deviations. 
SSC-OMP performs worse than the previous three SSC variants in every case.
%
LSR fails to achieve high accuracy across all
settings, and is it also characterised by a much higher performance variability when
$K$ is small. SSR, SMR and FGNSC have excellent performance when $K$ is small. However, their performance decreases with the increase of $K$, though 
the performance degrade more gradually for SMR and FGNSC than the previously discussed methods.
%
%
WSSR is the best performing method on this data set. It manages to achieve
median accuracy that is close to perfect, and very small performance variability for all values of $K$.

\subsection{WSSR+ Experiments}

In this section, we assess the performance of the proposed framework for constrained clustering in two cases. First, when
a random subset of labelled points is available
at the outset, and second when the
the active learning is used to select which points to label.
As in the previous section we use the MNIST and USPS data sets, and consider
different values of $K$.
For the constrained clustering problem, for each data set and for different
values of $\ncluster$, we obtain the labels of a proportion $p \in \{0.1,0.2,0.3\}$ 
of randomly selected points. 
We compare the performance of WSSR+ to that of Partition Level Constrained
Clustering (PLCC) \citep{liu2018partition} and Constrained Spectral
Partitioning (CSP) \citep{wang2014constrained}.
To ensure a fair comparison, the same initial affinity matrix is used in all
three constrained clustering algorithms. In particular the initial affinity
matrix is the one produced by WSSR.
PLCC involves one tuning parameter, $\lambda$, which controls the weight
assigned on the side information. Although in \cite{liu2018partition} it is
recommended to set $\lambda$ to be above 10,000 for stable performance, we
found that in our experiments this is a poor choice. Instead we sampled 20
random values for $\lambda$ in the range (0,1) and chose the one that
produced the highest clustering accuracy. CSP involves no tuning parameters.

Tables~\ref{tab_mnist_alcc} and~\ref{tab_usps_alcc} report the results for
MNIST and USPS respectively. The columns titled ``Med'' and ``Std'' report the
median and the standard deviation of clustering accuracy for the three
algorithms over 20 replications for different values of $p$.
In both tables there is an additional column labelled AL. This column reports
the performance on the same problems attained if the active learning described
in Section~\ref{sec_wssr_si} is used to select which points should be labelled,
rather than these being selected at random.
A comparison of this active learning 
strategy to those proposed by
\cite{lipor2015margin,lipor2017leveraging} is provided in 
\cite{peng2019subspace}.

\begin{table}[htbp!]
\begin{center}
\begin{tabular}{ lrcc cc cc c cc c}
\hline
\multirow{2}{*}{$K$}&\multirow{2}{*}{Pct.} &\multirow{2}{*}{WSSR}&\multicolumn{3}{c}{WSSR+} &\multicolumn{3}{c}{PLCC} &\multicolumn{3}{c}{CSP} \\
&&&AL&Med&Std&AL&Med&Std&AL&Med&Std\\
\hline
\multirow{3}{*}{$2$} 
&10\%&\multirow{3}{*}{1.00}&1.00&1.00&0.00&1.00&1.00&0.00&1.00&1.00&0.00\\
&20\%&&1.00&1.00&0.00&1.00&1.00&0.00&1.00&1.00&0.00\\
&30\%&&1.00&1.00&0.00&0.94&1.00&0.00&1.00&1.00&0.00\\
\hline
\multirow{3}{*}{$3$} 
&10\%&\multirow{3}{*}{1.00}&1.00&1.00&0.00&1.00&1.00&0.00&0.68&0.99&0.00\\
&20\%&&1.00&1.00&0.00&1.00&1.00&0.00&0.68&1.00&0.00\\
&30\%&&1.00&1.00&0.00&1.00&0.52&0.01&0.68&1.00&0.00\\
\hline
\multirow{3}{*}{$5$}
&10\%&\multirow{3}{*}{1.00}&1.00&1.00&0.00&1.00&0.99&0.00&0.79&0.44&0.17\\
&20\%&&1.00&1.00&0.00&1.00&0.99&0.00&0.80&0.64&0.12\\
&30\%&&1.00&1.00&0.00&0.98&0.45&0.01&0.80&0.99&0.07\\
\hline
\multirow{3}{*}{$8$} 
&10\%&\multirow{3}{*}{0.98}&0.99&0.98&0.00&0.86&0.62&0.06&0.97&0.44&0.17\\
&20\%&&0.99&0.98&0.00&0.86&0.55&0.05&0.99&0.65&0.11\\
&30\%&&0.99&0.99&0.00&0.85&0.80&0.05&0.99&0.89&0.09\\
\hline
\multirow{3}{*}{$10$}
&10\%&\multirow{3}{*}{0.98}&0.98&0.98&0.00&0.79&0.78&0.05&0.97&0.38&0.13\\
&20\%&&0.99&0.99&0.00&0.77&0.82&0.06&0.98&0.52&0.11\\
&30\%&&0.99&0.99&0.00&0.80&0.46&0.02&0.88&0.78&0.10\\
\hline
\end{tabular}
\end{center}
\caption{Clustering accuracy of various constrained clustering methods on the MNIST data. The initial affinity matrix for all methods is produced by WSSR. The column AL reports performance when the corresponding proportion of labelled
	points are selected through active learning rather than randomly.}
\label{tab_mnist_alcc}
\end{table}

\begin{table}[htbp!]
	\begin{center}
		\begin{tabular}{ lrcc cc cc c cc c}
			\hline
			\multirow{2}{*}{$K$}&\multirow{2}{*}{Pct.} &\multirow{2}{*}{WSSR}&\multicolumn{3}{c}{WSSR+} &\multicolumn{3}{c}{PLCC} &\multicolumn{3}{c}{CSP}\\
			&&&AL&Med&Std&AL&Med&Std&AL&Med&Std\\
			\hline
			\multirow{3}{*}{$2$} 
			&10\%&\multirow{3}{*}{1.00}&1.00&1.00&0.00&1.00&1.00&0.00&1.00&0.99&0.01\\
			&20\%&&1.00&1.00&0.00&1.00&1.00&0.00&1.00&1.00&0.00\\
			&30\%&&1.00&1.00&0.00&1.00&1.00&0.00&1.00&1.00&0.00\\
			\hline
			\multirow{3}{*}{$3$} 
			&10\%&\multirow{3}{*}{0.99}&0.99&0.99&0.00&1.00&0.99&0.00&0.99&0.98&0.13\\
			&20\%&&0.99&0.99&0.00&1.00&0.99&0.00&1.00&0.99&0.00\\
			&30\%&&1.00&0.99&0.00&1.00&0.52&0.02&1.00&0.99&0.00\\
			\hline
			\multirow{3}{*}{$5$}
			&10\%&\multirow{3}{*}{0.97}&0.98&0.97&0.00&0.97&0.97&0.00&0.98&0.35&0.11\\
			&20\%&&0.98&0.97&0.00&0.98&0.98&0.01&0.70&0.52&0.10\\
			&30\%&&0.98&0.98&0.00&0.89&0.45&0.01&0.98&0.80&0.13\\
			\hline
			\multirow{3}{*}{$8$} 
			&10\%&\multirow{3}{*}{0.97}&0.97&0.97&0.00&0.82&0.54&0.07&0.80&0.31&0.12\\
			&20\%&&0.98&0.97&0.00&0.81&0.50&0.04&0.96&0.47&0.11\\
			&30\%&&0.98&0.98&0.00&0.79&0.75&0.03&0.97&0.73&0.08\\
			\hline
			\multirow{3}{*}{$10$}
			&10\%&\multirow{3}{*}{0.97}&0.97&0.97&0.00&0.65&0.85&0.05&0.97&0.35&0.13\\
			&20\%&&0.97&0.97&0.00&0.73&0.85&0.06&0.96&0.49&0.11\\
			&30\%&&0.98&0.98&0.00&0.73&0.48&0.01&0.88&0.74&0.10\\
			\hline
		\end{tabular}
	\end{center}
\caption{Clustering accuracy of constrained clustering algorithms on the USPS data. The initial affinity matrix for all methods is produced by WSSR. The column AL reports performance when the corresponding proportion of labelled
	points are selected through active learning rather than randomly.}
	\label{tab_usps_alcc}
\end{table}

For the MNIST data set when $K$ is in the range $[2,5]$,
%
%
WSSR (without any label information) produces a perfect clustering.  For theses
cases, we inspect whether various constrained clustering algorithms retain this
performance after additional class information becomes available. WSSR+
accommodates the label information for all values of $p$ without degrading
accuracy on the rest of the data.  This is not the case for PLCC and CSP. The
performance difference between the constrained version of WSSR+ and the other
two constrained clustering algorithms becomes more pronounced when $K=8,10$.
The performance of neither CSP nor PLCC is guaranteed to increase as the
proportion of labelled points increases, and both algorithms exhibit higher
variability compared to WSSR+. The performance of PLCC and CSP is considerably
improved if the points whose label information is avaiable are chosen through
the active learning strategy of~\cite{peng2019subspace} compared to random
sampling.  This benefit is smallest for the WSSR+ since the algorithm achieves
near perfect accuracy for all values of $K$ on the MNIST data set.


Similar conclusions can be drawn from Table \ref{tab_usps_alcc} which reports
performance on the USPS data set. Again the initial cluster assignment produced by
the fully unsupervised WSSR algorithm is very accurate. Using a subset of
randomly chosen labelled points never reduces the performance of WSSR+,
whose performance is also very stable. This is not the case for CSP
and PLCC whose median performance can decline substantially and exhibit
considerable variability. As before using the active learning strategy
improves performance considerably for CSP and PLCC. Due to the almost
perfect performance of the constrained version of WSSR+ algorithm, the performance gain of using active learning in WSSR+ is marginal.

\section{Conclusions \& Future Work}\label{wssr_conclusions}

In this work, we proposed a subspace clustering method called Weighted Sparse Simplex Representation (WSSR), which relies on estimating
an affinity matrix by approximating each data point as a sparse
convex combination of nearby points.
We derived a lemma that provides a lower bound which can be used to select the only critical penalty
parameter in our formulation. Experimental results show that WSSR 
is competitive with state-of-the-art subspace clustering methods.
%
%
We also extended this approach to the problem of constrained clustering
(where external information about the actual
assignment of some points is available); and to active learning 
(in which case choosing which points will be queried is part of the learning
problem). The active learning approach we adopt aims to query points
whose label information can maximally improve the quality of the
overall cluster assignment.
The constrained clustering approach combines the strengths of spectral methods
and constrained $K$-subspace clustering and ensures that the resulting
clustering is consistent with all the label information available.
Experiments conducted on both synthetic and real data sets demonstrate the effectiveness
of our proposed methodology in all three scenarios: (a) subspace clustering without any side
information; (b) constrained subspace clustering with fixed side information; and (c) active learning.

%

%


In this work we focused on subspace clustering.  
For this problem the inverse cosine similarity is a natural 
proximity measure after projecting the data onto the unit sphere.
It would be interesting to explore other affinity measures that can potentially be used 
to capture data from affine subspaces and manifolds.

%
%


\begin{thebibliography}{35}
\providecommand{\natexlab}[1]{#1}
\providecommand{\url}[1]{{#1}}
\providecommand{\urlprefix}{URL }
\expandafter\ifx\csname urlstyle\endcsname\relax
  \providecommand{\doi}[1]{DOI~\discretionary{}{}{}#1}\else
  \providecommand{\doi}{DOI~\discretionary{}{}{}\begingroup
  \urlstyle{rm}\Url}\fi
\providecommand{\eprint}[2][]{\url{#2}}

\bibitem[{Basu et~al.(2008)Basu, Davidson, and Wagstaff}]{basu2008constrained}
Basu S, Davidson I, Wagstaff K (2008) Constrained clustering: Advances in
  algorithms, theory, and applications. CRC Press

\bibitem[{Boyd and Vandenberghe(2004)}]{boyd2004convex}
Boyd S, Vandenberghe L (2004) Convex optimization. Cambridge University Press

\bibitem[{Bradley and Mangasarian(2000)}]{bradley2000k}
Bradley PS, Mangasarian OL (2000) {$k$}-plane clustering. Journal of Global
  Optimization 16(1):23--32

\bibitem[{Critchley(1985)}]{critchley1985influence}
Critchley F (1985) Influence in principal components analysis. Biometrika
  72(3):627--636

\bibitem[{Dua and Graff(2017)}]{Dua:2019}
Dua D, Graff C (2017) {UCI} {M}achine {L}earning {R}epository.
  \urlprefix\url{http://archive.ics.uci.edu/ml}

\bibitem[{Elhamifar and Vidal(2013)}]{elhamifar2013sparse}
Elhamifar E, Vidal R (2013) Sparse subspace clustering: {A}lgorithm, theory,
  and applications. IEEE Transactions on Pattern Analysis and Machine
  Intelligence 35(11):2765--2781

\bibitem[{Gaines et~al.(2018)Gaines, Kim, and Zhou}]{gaines2018algorithms}
Gaines BR, Kim J, Zhou H (2018) Algorithms for fitting the constrained lasso.
  Journal of Computational and Graphical Statistics 27(4):861--871

\bibitem[{Hu et~al.(2014)Hu, Lin, Feng, and Zhou}]{hu2014smooth}
Hu H, Lin Z, Feng J, Zhou J (2014) Smooth representation clustering. In:
  Proceedings of the IEEE Conference on Computer Vision and Pattern
  Recognition, pp 3834--3841

\bibitem[{Huang et~al.(2013)Huang, Yan, Nie, Huang, Cai, Saykin, and
  Shen}]{huang2013new}
Huang H, Yan J, Nie F, Huang J, Cai W, Saykin AJ, Shen L (2013) A new sparse
  simplex model for brain anatomical and genetic network analysis. In:
  International Conference on Medical Image Computing and Computer-Assisted
  Intervention, Springer, pp 625--632

\bibitem[{Huang et~al.(2015)Huang, Nie, and Huang}]{huang2015new}
Huang J, Nie F, Huang H (2015) A new simplex sparse learning model to measure
  data similarity for clustering. In: 24th International Joint Conference on
  Artificial Intelligence

\bibitem[{Hull(1994)}]{hull1994database}
Hull JJ (1994) A database for handwritten text recognition research. IEEE
  Transactions on Pattern Analysis and Machine Intelligence 16(5):550--554

\bibitem[{LeCun et~al.(1998)LeCun, Bottou, Bengio, and
  Haffner}]{lecun1998gradient}
LeCun Y, Bottou L, Bengio Y, Haffner P (1998) Gradient-based learning applied
  to document recognition. Proceedings of the IEEE 86(11):2278--2324

\bibitem[{Li and Vidal(2015)}]{li2015structured}
Li C, Vidal R (2015) Structured sparse subspace clustering: A unified
  optimization framework. In: Proceedings of the IEEE Conference on Computer
  Vision and Pattern Recognition, pp 277--286

\bibitem[{Li et~al.(2017)Li, You, and Vidal}]{li2017structured}
Li C, You C, Vidal R (2017) Structured sparse subspace clustering: A joint
  affinity learning and subspace clustering framework. IEEE Transactions on
  Image Processing 26(6):2988--3001

\bibitem[{Li et~al.(2018{\natexlab{a}})Li, You, and Vidal}]{li2018geometric}
Li C, You C, Vidal R (2018{\natexlab{a}}) On geometric analysis of affine
  sparse subspace clustering. IEEE Journal of Selected Topics in Signal
  Processing 12(6):1520--1533

\bibitem[{Li et~al.(2018{\natexlab{b}})Li, Zhang, and Guo}]{li2018constrained}
Li C, Zhang J, Guo J (2018{\natexlab{b}}) Constrained sparse subspace
  clustering with side-information. In: 2018 24th International Conference on
  Pattern Recognition, IEEE, pp 2093--2099

\bibitem[{Lipor and Balzano(2015)}]{lipor2015margin}
Lipor J, Balzano L (2015) Margin-based active subspace clustering. In: 2015
  IEEE 6th International Workshop on Computational Advances in Multi-Sensor
  Adaptive Processing, IEEE, pp 377--380

\bibitem[{Lipor and Balzano(2017)}]{lipor2017leveraging}
Lipor J, Balzano L (2017) Leveraging union of subspace structure to improve
  constrained clustering. In: Proceedings of the 34th International Conference
  on Machine Learning, JMLR, vol~70, pp 2130--2139

\bibitem[{Liu et~al.(2012)Liu, Lin, Yan, Sun, Yu, and Ma}]{liu2012robust}
Liu G, Lin Z, Yan S, Sun J, Yu Y, Ma Y (2012) Robust recovery of subspace
  structures by low-rank representation. IEEE Transactions on Pattern Analysis
  and Machine Intelligence 35(1):171--184

\bibitem[{Liu et~al.(2018)Liu, Tao, and Fu}]{liu2018partition}
Liu H, Tao Z, Fu Y (2018) Partition level constrained clustering. IEEE
  Transactions on Pattern Analysis and Machine Intelligence 40(10):2469--2483

\bibitem[{Lu et~al.(2012)Lu, Min, Zhao, Zhu, Huang, and Yan}]{lu2012robust}
Lu C, Min H, Zhao Z, Zhu L, Huang D, Yan S (2012) Robust and efficient subspace
  segmentation via least squares regression. In: European Conference on
  Computer Vision, Springer, pp 347--360

\bibitem[{McWilliams and Montana(2014)}]{mcwilliams2014subspace}
McWilliams B, Montana G (2014) Subspace clustering of high-dimensional data: A
  predictive approach. Data Mining and Knowledge Discovery 28(3):736--772

\bibitem[{Ng et~al.(2002)Ng, Jordan, and Weiss}]{ng2002spectral}
Ng AY, Jordan MI, Weiss Y (2002) On spectral clustering: Analysis and an
  algorithm. In: Advances in Neural Information Processing Systems, pp 849--856

\bibitem[{Parikh and Boyd(2014)}]{parikh2014proximal}
Parikh N, Boyd S (2014) Proximal algorithms. Foundations and Trends in
  Optimization 1(3):127--239

\bibitem[{Peng and Pavlidis(2019)}]{peng2019subspace}
Peng H, Pavlidis NG (2019) Subspace clustering with active learning. In: IEEE
  International Conference on Big Data (Big Data), IEEE, pp 135--144

\bibitem[{Peng et~al.(2018)Peng, Pavlidis, Eckley, and
  Tsalamanis}]{peng2018subspace}
Peng H, Pavlidis NG, Eckley IA, Tsalamanis I (2018) Subspace clustering of very
  sparse high-dimensional data. In: IEEE International Conference on Big Data
  (Big Data), IEEE, pp 3780--3783

\bibitem[{Rao et~al.(2010)Rao, Tron, Vidal, and Ma}]{rao2009motion}
Rao S, Tron R, Vidal R, Ma Y (2010) Motion segmentation in the presence of
  outlying, incomplete, or corrupted trajectories. IEEE Transactions on Pattern
  Analysis and Machine Intelligence 32(10):1832--1845

\bibitem[{Vidal(2011)}]{vidal2011subspace}
Vidal R (2011) Subspace clustering. IEEE Signal Processing Magazine
  28(2):52--68

\bibitem[{Wagstaff et~al.(2001)Wagstaff, Cardie, Rogers, and
  Schr{\"o}dl}]{wagstaff2001constrained}
Wagstaff K, Cardie C, Rogers S, Schr{\"o}dl S (2001) Constrained $k$-means
  clustering with background knowledge. In: Proceedings of the 18th
  International Conference on Machine Learning, vol~1, pp 577--584

\bibitem[{Wang and Carreira-Perpin{\'a}n(2013)}]{wang2013projection}
Wang W, Carreira-Perpin{\'a}n MA (2013) Projection onto the probability
  simplex: An efficient algorithm with a simple proof, and an application.
  arXiv:13091541

\bibitem[{Wang and Davidson(2010)}]{wang2010flexible}
Wang X, Davidson I (2010) Flexible constrained spectral clustering. In:
  Proceedings of the 16th ACM SIGKDD International Conference on Knowledge
  Discovery and Data Mining, ACM, pp 563--572

\bibitem[{Wang et~al.(2014)Wang, Qian, and Davidson}]{wang2014constrained}
Wang X, Qian B, Davidson I (2014) On constrained spectral clustering and its
  applications. Data Mining and Knowledge Discovery 28(1):1--30

\bibitem[{Yang et~al.(2019)Yang, Liang, Wang, Rosin, and
  Yang}]{yang2019subspace}
Yang J, Liang J, Wang K, Rosin P, Yang MH (2019) Subspace clustering via good
  neighbors. IEEE Transactions on Pattern Analysis and Machine Intelligence

\bibitem[{You et~al.(2016)You, Robinson, and Vidal}]{you2016scalable}
You C, Robinson D, Vidal R (2016) Scalable sparse subspace clustering by
  orthogonal matching pursuit. In: Proceedings of the IEEE Conference on
  Computer Vision and Pattern Recognition, pp 3918--3927

\bibitem[{Zou and Hastie(2005)}]{zou2005regularization}
Zou H, Hastie T (2005) Regularization and variable selection via the elastic
  net. Journal of the Royal Statistical Society: Series B (Statistical
  Methodology) 67(2):301--320

\end{thebibliography}

\newpage
\appendix
\section{KKT Conditions for Optimality}\label{appd:kkt_cond}
In this section, we derive the Karush-Kuhn-Tucker (KKT) conditions for our proposed WSSR problem formulation. For any optimisation problem with differentiable objective and constraint functions for which strong duality holds, the KKT conditions are necessary and sufficient conditions for obtaining the optimal solution~\citep{boyd2004convex}. 

Firstly, the stationarity condition in the KKT conditions states that when optimality is achieved, the derivative of the Lagrangian with respect to $\bm{\beta}_{i}$ is zero. The Lagrangian $L\left(\bm{\beta}_{i};\lambda_{i},\bm{\mu}_{i}\right)$ associated with the WSSR problem in~\eqref{eq_wssr_en} can be expressed as
\begin{equation}
L(\bm{\beta}_{i};\lambda_{i},\bm{\mu}_{i})=
\frac{1}{2} \bm{\beta}_{i}^{\tp}\left(\Xmi^\tp \Xmi + \ridgepen D_\mathcal{I}^{\tp}D_\mathcal{I} \right)\bm{\beta}_{i}+\left(\rho \weight_{\mathcal{I}}-\Xmi^{\tp}\tx_{i} \right)^{\tp} \bm{\beta}_{i}
-\bm{\mu}_{i}^{\tp}\bm{\beta}_{i}+\lambda_{i}\left(\bm{\beta}_{i}^{\tp}\bm{1} - 1\right), 
\end{equation}
in which $\lambda_{i}$ is a scalar and $\bm{\mu}_{i}$ is a vector of non-negative Lagrange multipliers. 
Thus, the stationarity condition gives the following
\begin{equation}
\nabla L(\bm{\beta}_{i};\lambda_{i},\bm{\mu}_{i})=\left(\Xmi^\tp \Xmi + \ridgepen D_\mathcal{I}^{\tp}D_\mathcal{I} \right)\bm{\beta}_{i}-\Xmi^{\tp}\myx_{i}+\rho\weight_{\mathcal{I}}
+\lambda_{i}\bm{1}-\bm{\mu}_{i}=\bm{0},
\end{equation}
which can be simplified to
\begin{equation}
\bm{\beta}_{i} = \left(\Xmi^\tp \Xmi + \ridgepen D_\mathcal{I}^{\tp}D_\mathcal{I} \right)^{-1}\left(\Xmi^{\tp}\tx_{i} +\bm{\mu}_{i}-\rho \weight_{\mathcal{I}}+\lambda_{i} \bm{1} \right).
\end{equation}
Since all diagonal entries in $\Weight_{\mathcal{I}}$ are positive, the matrix $\left(\Xmi^\tp \Xmi + \ridgepen D_\mathcal{I}^{\tp}D_\mathcal{I}\right)$ is full rank thus invertible. 

Secondly, the KKT conditions state that any primal optimal $\bm{\beta}_{i}$ must satisfy both the equality and inequality constraints in~\eqref{eq_wssr_en}. In addition, any dual optimal $\lambda_{i}$ and $\bm{\mu}_{i}$ must satisfy the dual feasibility constraint $\bm{\mu}_{i} \geqslant \bm{0}$.
Thirdly, the KKT conditions state that $\mu_{ij}\beta_{ij}=0$ for all $j\in\mathcal{I}$ for any primal optimal $\bm{\beta}_{i}$ and dual optimal $\bm{\mu}_{i}$ when strong duality holds. This is called the complementary slackness condition.
To put everything together, when strong duality holds, any primal optimal $\bm{\beta}_{i}$ and any dual optimal $\lambda_{i}$ and $\bm{\mu}_{i}$ must satisfy the following KKT conditions:
\begin{align*}
& \text{Stationarity:  } \bm{\beta}_{i} = \left(\Xmi^\tp \Xmi + \ridgepen D_\mathcal{I}^{\tp}D_\mathcal{I} \right)^{-1}\left(\Xmi^{\tp}\tx_{i} +\bm{\mu}_{i}-\rho \weight_{\mathcal{I}}+\lambda_{i} \bm{1} \right),\\
& \text{Equality constraint:  } \bm{\beta}_{i}^{\tp}\bm{1} = 1, \\
&\text{Inequality constraint:  } \bm{\beta}_{i} \geqslant \bm{0},\\
& \text{Dual feasibility: } \bm{\mu}_{i} \geqslant \bm{0}, \\
& \text{Complementary slackness: } \mu_{ij}\beta_{ij}=0,\quad \forall\; j\in\mathcal{I}.
\label{eq_kkt}
\end{align*}

\section{Necessary and Sufficient Conditions for the Trivial Solution} \label{appd:nec_suf}
In Section \ref{sec_wssr_nec} and \ref{sec_wssr_suf}, we investigate the necessary and sufficient conditions under which the trivial solution is obtained. That is, only the most similar point is chosen and has coefficient one.

\subsection{Necessary Condition for the Trivial Solution}\label{sec_wssr_nec}
Consider the WSSR problem formulation in~\eqref{eq_wssr_en} for a given $\myx\in\mathcal{X}$, which we restate below:
\begin{equation}
\begin{aligned}
&\min_{\bm{\beta}_{i}}
& & \frac{1}{2} \myb_i^\tp \left( \Xmi^\tp \Xmi + \ridgepen \Weight_\mathcal{I}^{\tp}\Weight_\mathcal{I}\right) \myb_{i} + \left(\rho \weight_{\mathcal{I}}- \Xmi^\tp \tx_i\right)^\tp \myb_i\\
& \text{s.t.}
& & \bm{\beta}_{i}^{\tp}\bm{1} = 1,\quad \bm{\beta}_{i} \geqslant \bm{0}.
\end{aligned}
\label{eq_wssr_qp_copy}
\end{equation}
Without loss of generality, we assume that 
$\Xmi = \left[ \tx_{(1)}, \tx_{(2)}, \ldots, \tx_{(|\mathcal{I}|)}\right]$ 
where $\tx_{(k)}$~($k\in\left\{1,2,\ldots,|\mathcal{I}| \right\}$) is the $k$-th nearest neighbour of $\bx_{i}$ that lies on the perpendicular hyperplane of $\bx_{i}$. Similarly $\weight_{\mathcal{I}} = \text{diag}(\Weight_{\mathcal{I}})=\left[\weighti_{(1)}, \weighti_{(2)}, \ldots, \weighti_{(|\mathcal{I}|)}\right]^{\tp}$.
Let $\bm{\beta}_{i}^{\star}$ denote the optimal solution to~\eqref{eq_wssr_qp_copy}, we establish the necessary condition for the trivial solution that $\|\bm{\beta}_{i}^{\star}\|_{\infty}=1$ 
in Proposition~\ref{proof_wssr_nec}.

\begin{proposition}
	Assume the nearest neighbour of $\tx_{i}$ ($\tx_{i}=\bx_{i}$) is unique, i.e. $\tx_{i}^{(1)}\neq \tx_{i}^{(j)}$ for $(j)\neq (1)$. If the solution of the WSSR problem in~\eqref{eq_wssr_qp_copy} is given by $\bm{\beta}_{i}^\star = \bm{e}_1 = \left[1,0,\ldots,0 \right]^{\tp}\in\mathbb{R}^{|\mathcal{I}|}$, then the following holds
	\begin{align}
	\rho > \max\left\{0,
	\max_{j \in\left\{2,\ldots, |\mathcal{I}| \right\} } \frac{(\tx_{i}^{(1)} - \tx_{i}^{(j)})^\tp (\tx_{i}^{(1)} -\bx_{i}) + \ridgepen (\weighti_{i}^{(1)})^{2}}{\weighti_{i}^{(j)} - \weighti_{i}^{(1)}} \right\}.
	\label{eq:rho1}
	\end{align}
\end{proposition}

\begin{proof}
	To establish the above claim, it suffices to show that the directional
	derivative of the objective function at~$\bm{e}_1$ is positive for all feasible
	directions in the unit simplex~$\Delta^{|\mathcal{I}|}$. Without causing confusion, we drop the subscript $i$ in the following proof for ease of notation. Let us denote the objective function value in~\eqref{eq_wssr_qp_copy} as $f(\bm{\beta})$, then the derivative of the objective function is
	\begin{align*}
	\nabla f(\bm{\beta}) 
	& = \left(\Xmi^\tp \Xmi + \ridgepen \Weight_\mathcal{I}^{\tp}\Weight_\mathcal{I}\right) \bm{\beta} + \rho \weight_{\mathcal{I}} - \Xmi^{\tp} \tx = 
	H\bm{\beta} 
	+ \rho 
	\begin{bmatrix} \weighti^{(1)} \\ \weighti^{(2)} \\ \vdots \\  \weighti^{|\mathcal{I}|} 
	\end{bmatrix}
	- 
	\begin{bmatrix} 
	(\tx^{(1)})^{\tp} \tx \\ (\tx^{(2)})^{\tp} \tx \\ \vdots \\  (\tx^{|\mathcal{I}|})^{\tp} \tx 
	\end{bmatrix},
	\end{align*}
	where 
	\begin{align*}
	H=\begin{bmatrix}
	(\tx^{(1)})^{\tp} \tx^{(1)} + \ridgepen (\weighti^{(1)})^{2} & (\tx^{(1)})^{\tp} \tx^{(2)} & \dots & (\tx^{(1)})^{\tp} \tx^{|\mathcal{I}|} \\
	(\tx^{(2)})^{\tp} \tx^{(1)} & (\tx^{(1)})^{\tp} \tx^{(2)} + \ridgepen (\weighti^{(2)})^{2} & \dots & (\tx^{(2)})^{\tp} \tx^{|\mathcal{I}|} \\
	\vdots & \vdots & \ddots & \vdots \\
	(\tx^{|\mathcal{I}|})^{\tp} \myx^{(1)} & (\tx^{|\mathcal{I}|})^{\tp} \tx^{(2)} & \dots  & (\tx^{|\mathcal{I}|})^{\tp} \tx^{|\mathcal{I}|} + \ridgepen (\weighti^{|\mathcal{I}|})^{2} 
	\end{bmatrix}. 
	\end{align*}
	Therefore $\nabla{f}(\bm{e}_1)$ is equal to
	\begin{align*}
	\nabla f(\bm{e}_1) & =
	\begin{bmatrix} 
	(\tx^{(1)})^{\tp} \tx^{(1)} + \ridgepen (\weighti^{(1)})^{2} \\ (\tx^{(2)})^{\tp} \tx^{(1)} \\ \vdots \\ (\tx^{|\mathcal{I}|})^{\tp} \tx^{(1)} 
	\end{bmatrix}
	+ \rho 
	\begin{bmatrix} 
		\weighti^{(1)} \\ \weighti^{(2)} \\ \vdots \\  \weighti^{|\mathcal{I}|} 
	\end{bmatrix} 
	- \begin{bmatrix} 
	(\tx^{(1)})^{\tp} \tx \\ (\tx^{(2)})^{\tp} \tx \\ \vdots \\  (\tx^{|\mathcal{I}|})^{\tp} \tx 
	\end{bmatrix}.
	\end{align*}
	
	The directional derivative of $f$ at point $\bm{\beta}$ in the direction $\bm{e}_{j}$ is given by
	$\nabla{f}(\bm{\beta})^{\tp} \bm{e}_{j}$ for $j\in\left\{1,2,\ldots,|\mathcal{I}| \right\}$. To ensure that the directional derivative of $f$ at $\bm{e}_1$
	towards any feasible direction (that is any direction that retains
	$\bm{\beta}$ within the unit simplex $\Delta^{|\mathcal{I}|}$) is positive, it suffices to ensure that
	\begin{equation}
	\nabla{f}(\bm{e}_1)^{\tp} (\bm{e}_j - \bm{e}_1) > 0, \quad \forall \; j\in\left\{2,\ldots,|\mathcal{I}| \right\}.
	\end{equation}
	The above condition holds if the following holds
	\begin{equation}
	\rho > \max_{j\in\left\{2,\ldots,|\mathcal{I}| \right\}} \frac{(\tx^{(1)} - \tx^{(j)})^{\tp} (\tx^{(1)} -\tx) + \ridgepen (\weighti^{(1)})^{2}}{\weighti^{(j)} - \weighti^{(1)}}.
	\label{eq_rho}
	\end{equation}
\eqref{eq:rho1} is obtained by combining the above inequality with
	the requirement that $\rho \geq 0$.
\end{proof}

\subsection{Sufficient Condition for the Trivial Solution}\label{sec_wssr_suf}
Next, we show that~\eqref{eq_rho} is a sufficient condition for the trivial solution $\bm{\beta}_{i}^{\star} = \bm{e}_{1}$.

\begin{proposition}\label{proof_wssr_suff}
	Assume the nearest neighbour of $\tx_{i}$ ($\tx_{i}=\bx_{i}$) is unique, i.e. $\tx_{i}^{(1)}\neq \tx_{i}^{(j)}$ for $(j)\neq (1)$. If the following holds
	\begin{align}
	\rho > 
	\max_{j \in\left\{2,\ldots, |\mathcal{I}| \right\} } \frac{(\tx_{i}^{(1)} - \tx_{i}^{(j)})^{\tp} (\tx_{i}^{(1)} -\tx_{i}) + \ridgepen (\weighti_{i}^{(1)})^{2}}{\weighti_{i}^{(j)} - \weighti_{i}^{(1)}},
	\label{eq_rho_copy}
	\end{align}
	then the solution to~\eqref{eq_wssr_qp_copy} is given by $\bm{\beta}_{i}^{\star} = \bm{e}_1$.
	In addition, if for all $j\in\left\{2,\ldots, |\mathcal{I}| \right\}$ we have \[\left(\bx_{i}^{(j)}-\bx_{i}^{(1)} \right)^{\tp}\left(\bx_{i}-\bx_{i}^{(1)} \right)\leqslant 0,\] 
	then the solution to~\eqref{eq_wssr_qp_copy} is given by $\bm{\beta}_{i}^\star = \bm{e}_1$ for all $\rho>0$.
\end{proposition}
\begin{proof}
	For the first part of the proposition, if~\eqref{eq_rho_copy} holds, then for all $j$ ($j \in\left\{2,\ldots,|\mathcal{I}| \right\}$) we have
	\begin{align*}
	&\rho > \frac{(\tx_{i}^{(1)} - \tx_{i}^{(j)})^{\tp} (\tx_{i}^{(1)} -\tx_{i}) + \ridgepen (\weighti_{i}^{(1)})^{2}}{\weighti_{i}^{(j)} - \weighti_{i}^{(1)}} \\
	\Leftrightarrow\quad & (\tx_{i}^{(j)})^{\tp}\left(\tx_{i}^{(1)}-\tx_{i} \right)+\rho \weighti_{i}^{(j)} > (\tx_{i}^{(1)})^{\tp}\left(\tx_{i}^{(1)}-\tx_{i} \right)+\ridgepen (\weighti_{i}^{(1)})^{2}+\rho\weighti_{i}^{(1)}\\
	\Leftrightarrow\quad &
	\nabla f(\bm{e}_{1})^{\tp}\bm{e}_{j}>\nabla f(\bm{e}_{1})^{\tp}\bm{e}_{1}\\
	\Leftrightarrow\quad &
	\nabla f(\bm{e}_{1})^{\tp}\left(\bm{e}_{j}-\bm{e}_{1} \right)>0.  
	\end{align*}
	The last line from above means that the directional derivative at $\bm{e}_1$ towards any other feasible direction within the unit simplex $\Delta^{(N-1)}$ is positive. Thus the solution to~\eqref{eq_wssr_qp_copy} is given by $\bm{\beta}^\star = \bm{e}_1$.
	
	For the second part of the proposition, we first provide a geometric interpretation in Figure~\ref{fig_prop2_illu} for the meaning of the statement. 
	\begin{figure}[htbp!]
		\centering
		\includegraphics[width=.5\textwidth]{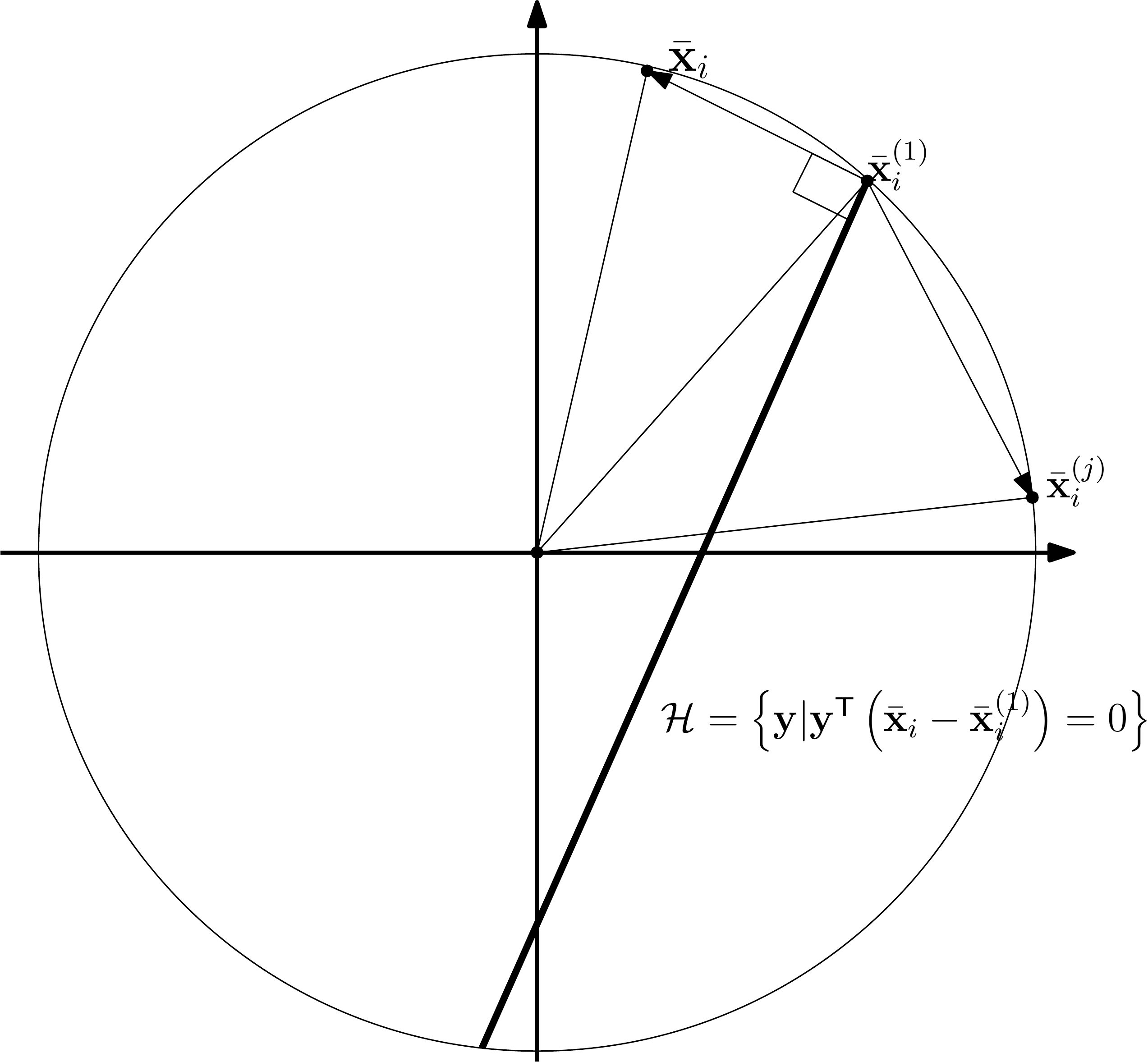}
		\caption{A geometric interpretation for when the trivial solution is obtained.}
		\label{fig_prop2_illu}
	\end{figure}
	In Figure~\ref{fig_prop2_illu}, $\bx_{i}$ is the point to be approximated and $\bx_{i}^{(1)}$ is its nearest neighbour on the unit sphere. The bold black line is the perpendicular hyperplane of $\left(\bx_{i}-\bx_{i}^{(1)} \right)$, which is denoted by $\mathcal{H}=\left\{\bm{y}| \bm{y}^{\tp}\left(\bx_{i}-\bx_{i}^{(1)} \right)=0  \right\}$. 
	
	Assume that all points apart from $\bx_{i}$ and $\bx_{i}^{(1)}$ lie on one side of the hyperplane~$\mathcal{H}$, opposite the side to which $\bx_{i}$ resides in. That is, $\left(\bx_{i}^{(j)}-\bx_{i}^{(1)} \right)^{\tp}\left(\bx_{i}-\bx_{i}^{(1)} \right)\leqslant 0$ for all $j\in\left\{2,\ldots,|\mathcal{I}| \right\}$. We can see that $\mathcal{H}$ is a supporting hyperplane for $\text{conv}(\bar{\mathcal{X}}\backslash \left\{\bx_{i}\right\})$. Consider for an arbitrary point $\bm{y}=Y\bm{\beta}\in\text{conv}(\bar{\mathcal{X}}\backslash \left\{\bx_{i}\right\})$, we have
	\begin{align*}
		&\left(\bx_{i}-\bx_{i}^{(1)} \right)^{\tp}\left(Y\bm{\beta}-\bx_{i}^{(1)} \right)\\
		=&\left(\bx_{i}-\bx_{i}^{(1)} \right)^{\tp}\left[\sum_{j=1}^{|\mathcal{I}|} \beta_{j}\left(\bx_{i}^{(j)}-\bx_{i}^{(1)} \right)  \right]\\
		=&\sum_{j=1}^{|\mathcal{I}|} \beta_{j}\left(\bx_{i}-\bx_{i}^{(1)} \right)^{\tp}
		\left(\bx_{i}^{(j)}-\bx_{i}^{(1)} \right) \\
		\leqslant &0.
	\end{align*} 
	That is, all points apart from $\bx_{i}$ and $\bx_{i}^{(1)}$ lie on one side of the supporting hyperplane $\mathcal{H}=\left\{\bm{y}| \bm{y}^{\tp}\left(\bx_{i}-\bx_{i}^{(1)} \right)=0  \right\}$. That is, $\left(\bx_{i}^{(j)}-\bx_{i}^{(1)} \right)^{\tp}\left(\bx_{i}-\bx_{i}^{(1)} \right)\leqslant 0$. In this case, any linear combination of the column vectors in $Y$ would be further away from $\bx_{i}$ than using $\bx_{i}^{(1)}$ itself as the approximation.

	Therefore, the proposition says if $\left(\bx_{i}^{(j)}-\bx_{i}^{(1)} \right)^{\tp}\left(\bx_{i}-\bx_{i}^{(1)} \right)\leqslant 0$ is satisfied for all $j\in\left\{2,\ldots, |\mathcal{I}| \right\}$, then the trivial solution can be obtained for any $\rho>0$. 
\end{proof}

\section{WSSR+ Experiments on UCI Benchmark Data}
In this section, we conduct further experiments to compare WSSR+ with other state-of-the-art constrained clustering methods on data sets that do not exhibit subspace structure.
The experiments are conducted on four UCI benchmark data sets \citep{Dua:2019}, which have been used previously to demonstrate the effectiveness of constrained spectral clustering methods \citep{wang2014constrained, liu2018partition}. 
A summary of the data characteristics can be found in Table \ref{tab_uci}.
\begin{table}[h!]
	\begin{center}
		\begin{tabular}{cccc}
			\hline
			Data sets& No. of points ($N$)& No. of features ($P$)&No. of clusters ($K$)\\
			\hline
			iris&150&4&3\\
			wine&178&13&3\\
			ecoli&336&343&8\\
			glass&214&9&6\\
			\hline
		\end{tabular}
	\end{center}
	\caption{A summary of the UCI benchmark data sets.}
	\label{tab_uci}
\end{table}

Performance results in terms of clustering accuracy are reported in Table \ref{tab:acsr_exp1} under varying proportions of side information. For each proportion $p\%$ of side information, we run all experiments for 20 replications and report the median clustering accuracy along with the standard deviation. For each data set, the best performance results are highlighted in bold, and the second best performance results are underlined.

\begin{table}[h!]
	\begin{center}
		\begin{tabular}{c ccc cc cc c cc}
			\hline
			\multirow{2}{*}{Data} &\multirow{2}{*}{Pct.} &\multirow{2}{*}{WSSR}&\multicolumn{2}{c}{WSSR+} &\multicolumn{2}{c}{PLCC} &\multicolumn{2}{c}{CSP} &\multicolumn{2}{c}{LCVQE} \\
			&&&Med&Std&Med&Std&Med&Std&Med&Std\\
			\hline
			\multirow{3}{*}{iris} 
			&10\%&\multirow{3}{*}{0.97}&\textbf{0.97}&0.00&\textbf{0.97}&0.00&\underline{0.94}&0.14&0.90&0.01\\
			&20\%&&\underline{0.97}&0.00&\textbf{0.98}&0.01&\textbf{0.98}&0.07&0.92&0.02\\
			&30\%&&\underline{0.98}&0.01&\textbf{0.99}&0.01&\underline{0.98}&0.12&0.93&0.01\\
			\hline
			\multirow{3}{*}{wine} 
			&10\%&\multirow{3}{*}{0.83}&\textbf{0.86}&0.02&\underline{0.85}&0.01&0.68&0.06&0.71&0.02\\
			&20\%&&\textbf{0.88}&0.02&\underline{0.86}&0.01&0.75&0.07&0.73&0.03\\
			&30\%&&\textbf{0.88}&0.02&\underline{0.87}&0.04&0.85&0.05&0.71&0.04\\
			\hline
			\multirow{3}{*}{ecoli} 
			&10\%&\multirow{3}{*}{0.78}&\textbf{0.77}&0.01&0.67&0.06&\underline{0.76}&0.03&\textbf{0.77}&0.03\\
			&20\%&&\underline{0.80}&0.01&\textbf{0.82}&0.02&0.76&0.02&\underline{0.80}&0.05\\
			&30\%&&\textbf{0.81}&0.02&0.71&0.05&0.76&0.02&\underline{0.80}&0.03\\
			\hline
			\multirow{3}{*}{glass} 
			&10\%&\multirow{3}{*}{0.68}&\textbf{0.69}&0.01&\underline{0.67}&0.04&0.65&0.09&0.58&0.02\\
			&20\%&&\textbf{0.69}&0.01&\underline{0.66}&0.01&\textbf{0.69}&0.04&0.60&0.04\\
			&30\%&&\underline{0.70}&0.01&0.60&0.03&\textbf{0.72}&0.04&0.61&0.04\\
			\hline
		\end{tabular}
	\end{center}
	\caption{Clustering accuracy of various spectral-based constrained clustering methods on UCI benchmark data sets.}
	\label{tab:acsr_exp1}
\end{table}

It is worth noting that all the initial clustering accuracy scores achieved by WSSR here are significantly better than that produced by spectral clustering as reported in \cite{liu2018partition}. As a result, the performance results for all three competing methods are much better than what have been previously reported.
With that said, it can be seen that extra side information can have a negative impact on the resulting accuracy. For example, this is the case for PLCC on both \emph{ecoli} and \emph{glass} data sets, in which the performance degrades with the increase of side information.

Overall, WSSR+ has a favourable performance against all three competing methods. In particular, it enjoys the best performance across all side information levels on the \emph{wine} data set and remains one of the top two performers on the remaining data sets. As opposed to the adverse effects that have been observed in the competing methods, the performance of WSSR+ improves consistently with the increase of side information. 
The performance of PLCC appears to be similarly competitive to that of WSSR+. However the standard deviation from PLCC can be comparatively big on data sets such as \emph{ecoli} and \emph{glass}. The performance variability in PLCC undermines its reliability and seemingly high median clustering accuracy. 

To further investigate the stability of various methods, we provide detailed performance visualisations for all methods on all data sets in Figure \ref{fig_uci_vis}. Each plot presents the minimum, median, and maximum performance of each constrained clustering method across 20 replications. The proportion of known class labels $p\%$ range from 0.1 to 1.0, with 1.0 being all class labels are known.  
High variability can be observed from the performance of CSP (on \emph{iris} data set) and LCVQE (on \emph{wine} and \emph{glass} data sets). It also becomes obvious that even for methods that have relatively small variability, such as PLCC and WSSR+, consistent performance improvement is not always achieved on all data sets. In particular, the clustering accuracy of PLCC decreased when the available side information increased to 40\% on the iris data. Taking into account both consistent and stable performance improvement, WSSR+ is competitive against these state-of-the-art methods.

\begin{figure}[ht!]
	\centering
	\includegraphics[width=.48\linewidth, height=.4\textwidth, trim={4cm 10cm 4cm 10cm}]{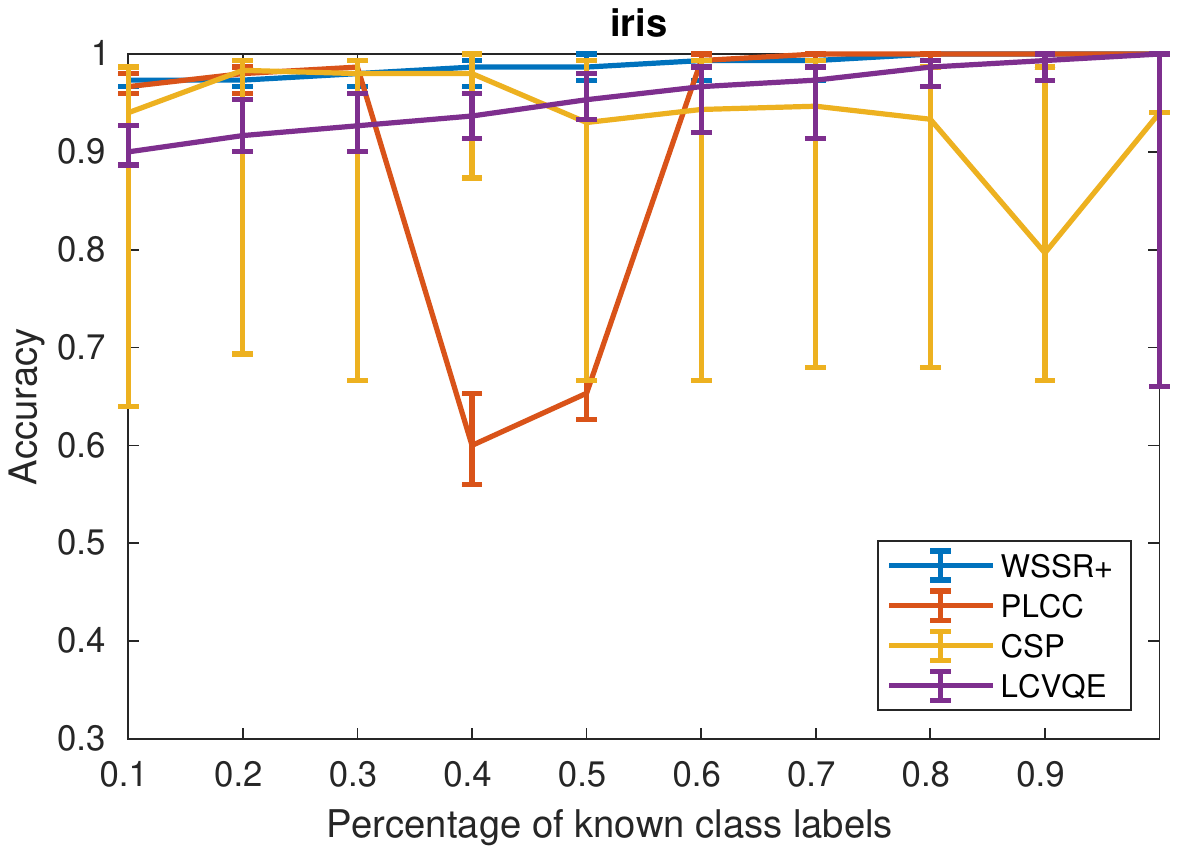}
	\includegraphics[width=.48\linewidth, height=.4\textwidth, trim={4cm 10cm 4cm 10cm}]{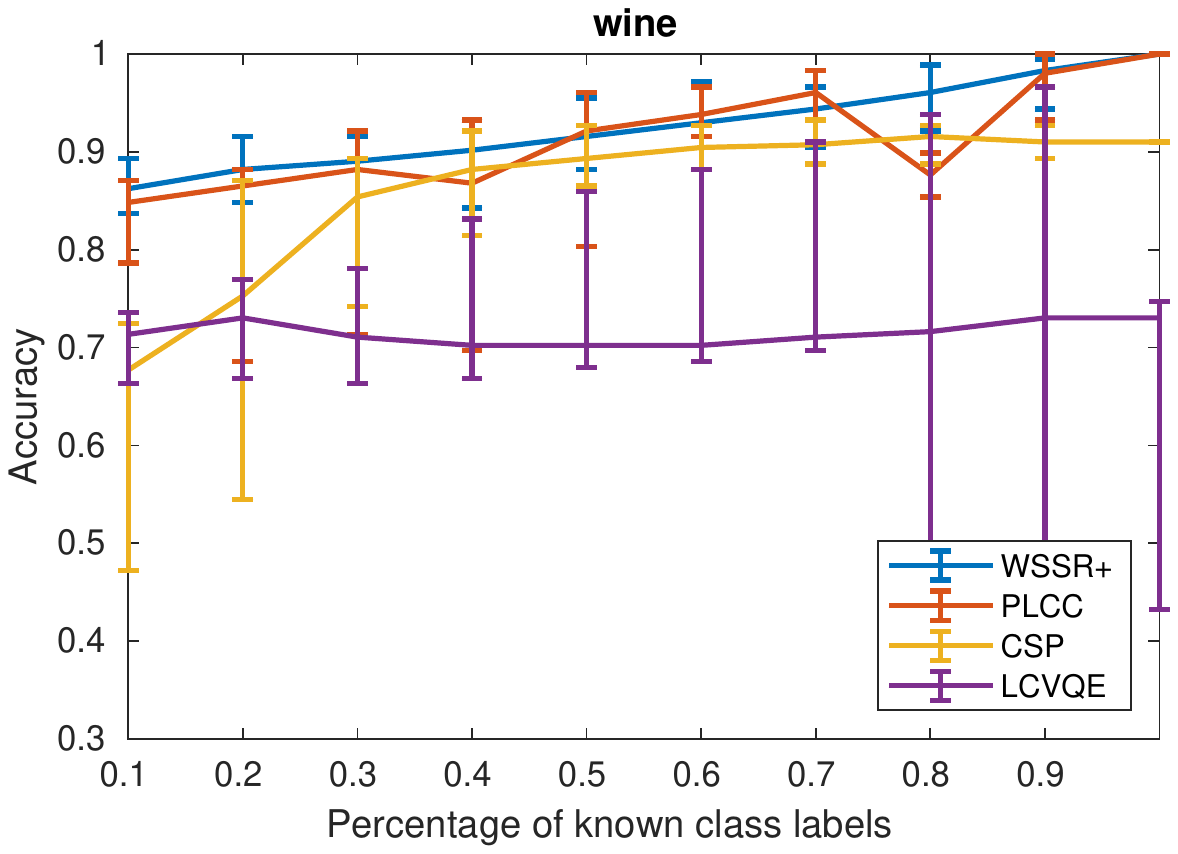}
	\includegraphics[width=.48\linewidth, height=.4\textwidth, trim={4cm 10cm 4cm 10cm}]{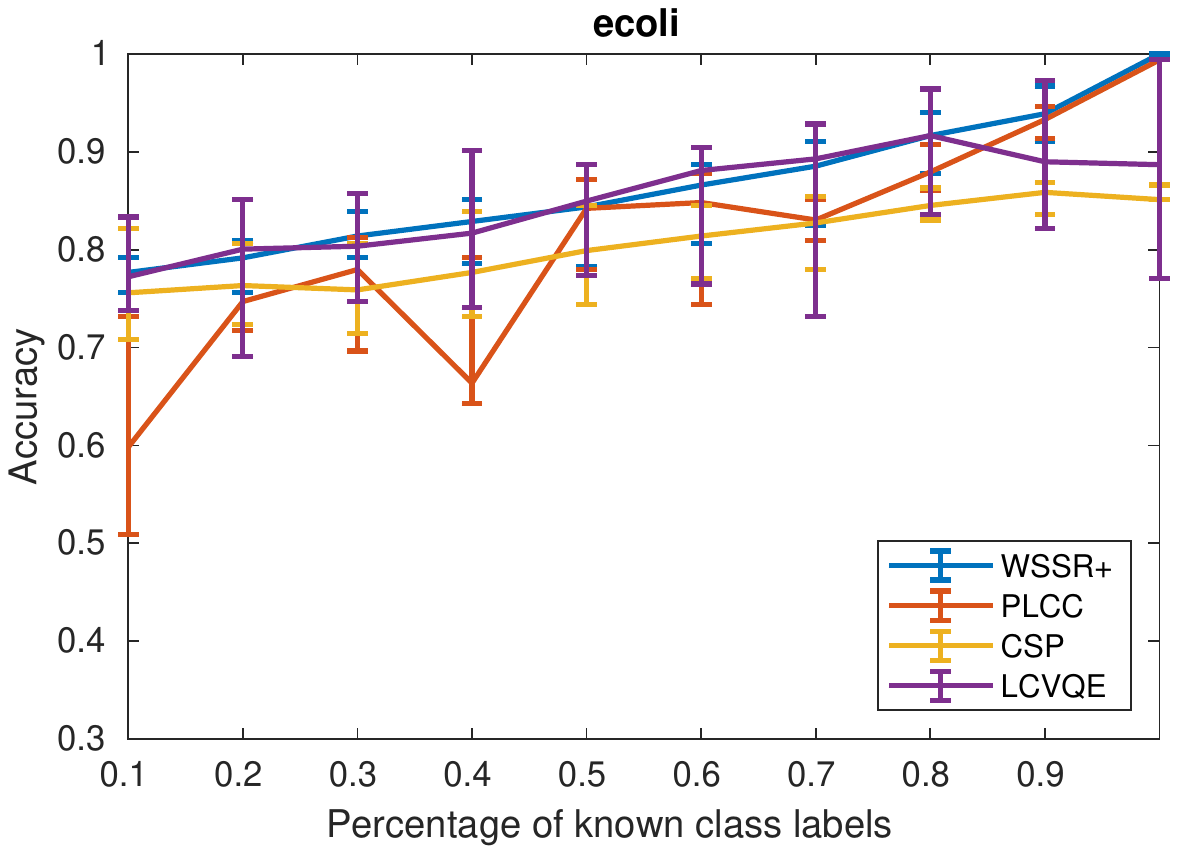}
	\includegraphics[width=.48\linewidth, height=.4\textwidth, trim={4cm 10cm 4cm 10cm}]{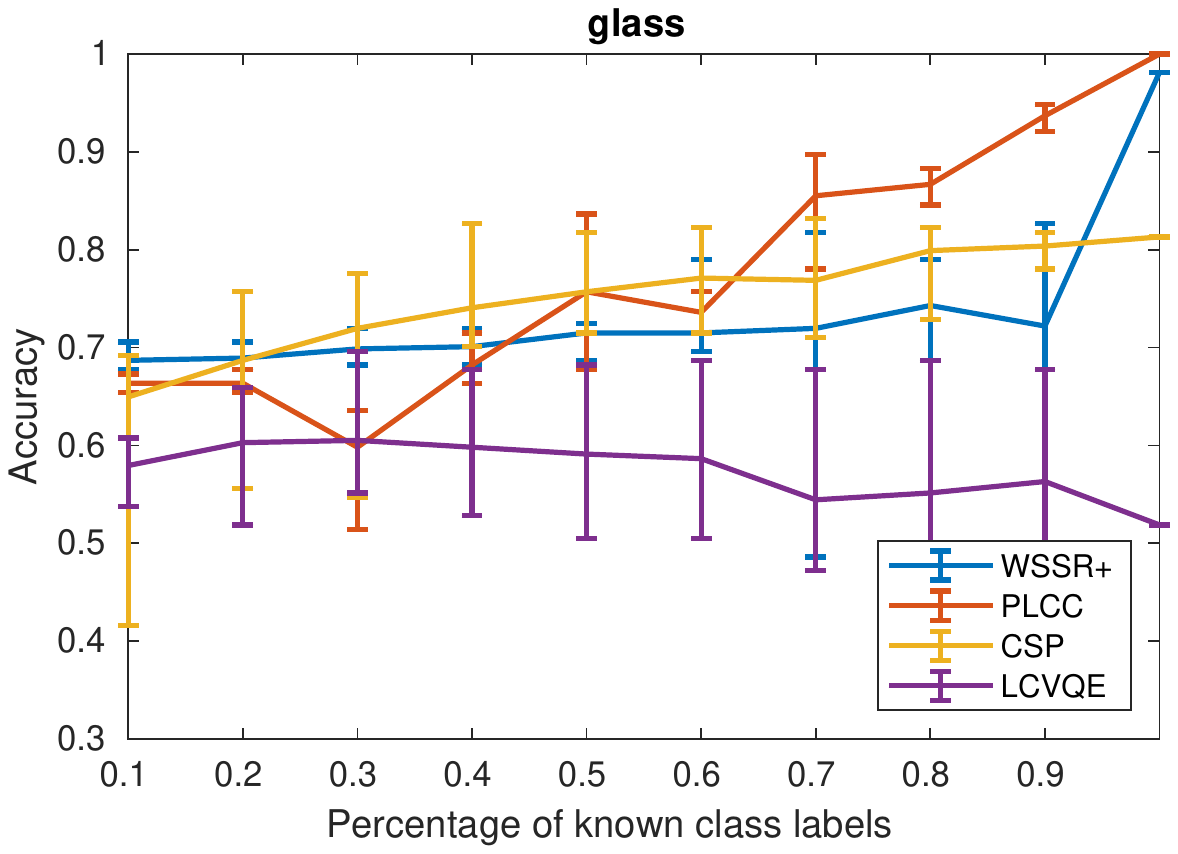}
	\caption{The clustering accuracy (min, median, max) of various constrained clustering algorithms over 20 replications.}
	\label{fig_uci_vis}
\end{figure}

\end{document}